\documentclass[11pt]{article}
\usepackage[utf8]{inputenc}
\usepackage[T1]{fontenc}
\usepackage{amsthm, amsmath}
\usepackage{amsfonts}
\usepackage[margin=1in]{geometry}
\usepackage[algo2e,ruled,noend,resetcount,linesnumbered]{algorithm2e}
\usepackage{graphicx}
\usepackage{comment}
\usepackage{nicematrix}
\usepackage{tikz}
\usepackage[shortlabels]{enumitem}
\usepackage{thmtools}
\usepackage{thm-restate}
\usepackage{algorithm, algorithmicx, algpseudocode} 
\usepackage{array}
\usepackage{stmaryrd}
\usepackage[colorlinks,citecolor=blue,linkcolor=blue,urlcolor=red,pagebackref]{hyperref}
\usepackage[capitalise]{cleveref}
\usepackage{xspace}
\usepackage{caption}
\usepackage{subcaption}
\usepackage{derivative}

\theoremstyle{plain}
\newtheorem{theorem}{Theorem}[section]
\newtheorem{corollary}[theorem]{Corollary}

\newtheorem{lemma}[theorem]{Lemma}
\newtheorem{claim}[theorem]{Claim}

\newtheorem{hypothesis}[theorem]{Hypothesis}

\theoremstyle{definition}
\newtheorem{definition}[theorem]{Definition}

\theoremstyle{remark}

\newcommand{\cA}{\mathcal{A}}
\newcommand{\cB}{\mathcal{B}}
\newcommand{\cT}{\mathcal{T}}

\newcommand{\R}{\mathbb{R}}
\newcommand{\Z}{\mathbb{Z}}

\newcommand{\inv}{^{-1}}

\newcommand{\OV}{\textsf{OV}}

\newcommand{\sat}{\textsf{SAT}}
\newcommand{\TMUL}{\mathsf{T_{MUL}}}
\newcommand{\TATTC}{\mathsf{T_{ATTC}}}
\newcommand{\mip}{\mathsf{Max}\textup{-}\mathsf{IP}}

\newcommand{\SETH}{\textsf{SETH}}

\newcommand{\SET}[1]{\left\{#1\right\}}
\newcommand{\otherwise}{\mathrm{ o/w }}

\newcommand{\tO}[1]{\tilde{O}(#1)}
\newcommand{\poly}{\mathrm{ poly}}
\newcommand{\polylog}{\mathrm{ polylog}}
\newcommand{\bigO}[1]{O\left(#1\right)}
\newcommand{\bigTh}[1]{\Theta\left(#1\right)}
\newcommand{\bigtO}[1]{\tilde{O}\left(#1\right)}

\newcommand{\ptime}{\textbf{ P}}
\newcommand{\np}{\textbf{ NP}}

\newcommand*{\horzbar}{\rule[.5ex]{2.5ex}{0.5pt}}

\newcommand{\onevec}{\mathbf{1}}
\newcommand{\zerovec}{\mathbf{0}}
\newcommand{\transpose}{\top}

\newcommand{\diag}{\mathrm{diag}}

\newcommand{\rank}{\mathrm{rank}}

\newcommand{\norm}[1]{\left\lVert#1\right\rVert}

\newcommand{\att}{\mathrm{Attention}}
\newcommand{\attc}{\mathsf{AttC}}
\newcommand{\aattc}{\mathsf{AAttC}}
\newcommand{\aattlgc}{\mathsf{AAttLGC}}

\newcommand{\vecAttF}{\mathrm{VectorAttention}}
\newcommand{\appAttF}{\mathrm{ApproxAttention}}

\newcommand{\rsds}{\mathbf{RSDS}}

\newcommand{\relevant}{\mathrm{relevant}}
\newcommand{\irrelevant}{\mathrm{irrelevant}}

\newcommand{\entryBound}{B}
\newcommand{\polyBound}{W}

\title{Subquadratic Algorithms and Hardness for Attention with Any Temperature}
\author{Shreya Gupta\thanks{University of California, San Diego. The authors are partially supported by NSF grants 1652303, 1909046, 2112533, and HDR TRIPODS Phase II grant 2217058.} \and Boyang Huang\footnotemark[1] \and Barna Saha\footnotemark[1] \and  Yinzhan Xu\footnotemark[1] \and Christopher Ye\footnotemark[1]}

\begin{document}

\maketitle

\begin{abstract}

Despite the popularity of the Transformer architecture, the standard algorithm for computing Attention suffers from quadratic time complexity in context length $n$. Alman and Song [NeurIPS 2023] showed that when the head dimension $d = \Theta(\log n)$, subquadratic Attention is possible if and only if the inputs have small entries bounded by $B = o(\sqrt{\log n})$ in absolute values, under the Strong Exponential Time Hypothesis ($\mathsf{SETH}$). Equivalently, subquadratic Attention is possible if and only if the softmax is applied with high temperature for $d=\Theta(\log n)$. Running times of these algorithms depend exponentially on $B$ and thus they do not lead to even a polynomial-time algorithm outside the specific range of $B$. 

This naturally leads to the question: when can Attention be computed efficiently without strong assumptions on temperature? Are there fast attention algorithms that scale polylogarithmically with entry size $B$? In this work, we resolve this question and characterize when fast Attention for arbitrary temperatures is possible. First, for all constant $d = O(1)$, we give the first subquadratic $\tilde{O}(n^{2 - 1/d} \cdot \mathrm{polylog}(B))$ time algorithm for Attention with large $B$. Our result holds even for matrices with large head dimension if they have low rank. In this regime, we also give a similar running time for Attention gradient computation, and therefore for the full LLM training process. Furthermore, we show that any substantial improvement on our algorithm is unlikely. In particular, we show that even when $d = 2^{\Theta(\log^* n)}$,  Attention requires $n^{2 - o(1)}$ time under $\mathsf{SETH}$.

Finally, in the regime where $d = \mathrm{poly}(n)$, the standard algorithm requires $O(n^{2} d)$ time while previous lower bounds only ruled out algorithms with truly subquadratic time in $n$. We close this gap and show that the standard algorithm is optimal under popular fine-grained complexity assumptions. 

\end{abstract}

\section{Introduction}

Large Language Models powered by the Transformer architecture \cite{DBLP:conf/nips/VaswaniSPUJGKP17} have been at the heart of modern AI revolution completely reshaping the landscapes of natural language processing, computer vision, and multitude of other applications.
The Attention mechanism is the cornerstone of the Transformer architecture. Attention computes correlations between different tokens of the sequences, allowing Transformers to model dependencies regardless of the position of the tokens in the sequences.
Despite its popularity, standard algorithms for computing Attention require quadratic time complexity, as they compute the Attention matrix explicitly.

Formally, the Attention mechanism is defined as follows.
Let $Q \in \R^{n \times d}$ be the query matrix, $K \in \R^{n \times d}$ the key matrix, and $V \in \R^{n \times d}$ the value matrix.
We call $n$ the sequence or context length and $d$ the head dimension.
The Attention matrix is the $n \times n$ matrix obtained by applying softmax\footnote{Given a vector $x$, applying softmax to $x$ replaces $x_{i}$ with $\exp(x_{i}) / \sum_{j} \exp(x_{j})$.} to each row of $Q K^{\transpose}$.
Each entry in the matrix represents the attention weight between a particular input token (query token $Q$) and output token (key token $K$).
Finally, the Attention mechanism asks for the product of the Attention matrix with the value matrix $V$.

We give the formal definition below.
Note that $\exp(X)$ applies $\exp$ to each entry of a matrix $X$.

\begin{restatable}[Attention]{definition}{AttDefn}
\label{defn:attention}
    Given input matrices $Q, K, V \in \R^{n \times d}$,
    \emph{Attention} on $Q, K, V$ is defined $\att(Q, K, V) := D\inv A V \in \R^{n \times d}$ where $A := \exp(Q K^\top)$\footnote{In practice, a scaled dot-product attention, defined as $A := (QK^\top / \sqrt{d})$, is also commonly used for training efficiency \cite{DBLP:conf/nips/VaswaniSPUJGKP17}.} and $D := \diag(A \onevec)$.
\end{restatable}

In practice, there is an input $X \in \R^{n \times d}$ and weight matrices $W_{Q}, W_{K}, W_{V} \in \R^{d \times d}$ such that $Q = X W_{Q}, K = X W_{K}, V = X W_{V}$. 
Since $Q, K, V$ can be computed from $X, W_{Q}, W_{K}, W_{V}$ in $O(n d^{2})$ time, we assume for simplicity that the inputs $Q, K, V$ are given directly.

Typically, it suffices to \emph{approximately} perform Attention computations. 
In particular, it is not necessary (or even reasonable) to expect Attention to be computed exactly due to
the softmax operation.
Thus, we study Approximate Attention as defined in \cite{alman2023fastattentionboundedentries}.
Typically, it suffices to compute the result up to polynomial precision (i.e. inverse polynomial additive error). 

\begin{definition}[Approximate Attention Computation $\attc(n, d, B, \varepsilon)$]
\label[definition]{defn:attention-computation-problem}
    Given matrices $Q, K, V \in [-\entryBound, \entryBound]^{n \times d}$ and $B, \varepsilon > 0$, compute $O \in \R^{n \times d}$ such that $\norm{O - \att(Q, K, V)}_{\infty} < \varepsilon$.
\end{definition}

The standard (and most widely used) algorithm for Attention (even in approximate form) requires quadratic time.
The algorithm begins by explicitly computing matrix product $QK^{\transpose}$, applies softmax to obtain $D^{-1} A$ and then computes the matrix product $(D^{-1} A) V$.
Using standard matrix multiplication, this requires $O(n^{2} d)$ time.
Even ignoring computation time of matrix multiplication, explicitly computing the $A$ matrix already requires $O(n^{2})$ time.

However, the inputs (and outputs) only have size $O(n d)$.
Indeed, an algorithm that does not compute $A$ explicitly could compute Attention in $O(n d)$ time, incurring only \emph{linear} dependence on the context length $n$.
This leads to the fundamental question concerning the complexity of Attention.

\begin{center}
{\em Question 1: When can Attention be computed faster than $n^2 d$ time?}
\end{center}

Towards answering this question, \cite{alman2023fastattentionboundedentries} showed that for $d = \Theta(\log n)$, Attention can be computed in $n^{1 + o(1)}$ time whenever $B = o(\sqrt{\log n})$.
Furthermore, whenever $B = \Omega(\sqrt{\log n})$ and $d = \Theta(\log n)$, Attention requires $n^{2 - o(1)}$ time under $\SETH$, a popular hardness hypothesis.

Yet there remain several shortcomings in our current understanding of Attention.
Fast algorithms for Attention are only known for inputs with small entries (i.e.\ $B = o(\sqrt{\log n})$).
Such a strong bound on the entries of $Q, K$ essentially restricts the Attention mechanism to use softmax with high temperature (enforcing a near-uniform distribution on the key and query matrices).
Temperature, denoted by $T$, is a key hyperparameter for Attention that dictates how random the output is. 
Formally, Attention with temperature $T$ replaces $A := \exp(QK^{\transpose})$ with $A := \exp(QK^{\transpose}/T)$ so that high temperature corresponds to high entropy (more likely to select keys with lower scores).
In many tasks, temperature is a key hyperparameter with potentially significant impact on accuracy and stability~\cite{DBLP:journals/tmlr/AgarwalaSPD23, temperature2}.
In contrastive learning, temperature has been found to significantly impact both the accuracy \cite{DBLP:conf/icml/ChenK0H20, DBLP:conf/cvpr/WangL21a, DBLP:conf/cvpr/Hu00Q21} as well as the learned representations \cite{DBLP:conf/icml/WangIsola20, DBLP:conf/cvpr/WangL21a, DBLP:conf/nips/RobinsonSYBJS21} of the model.
Dynamically varying temperature throughout the training process can also help balance multiple training objectives \cite{DBLP:journals/tbbis/KhaertdinovAG22, DBLP:conf/iclr/KuklevaBSK023, MannaTemperature}.
In instances where low entropy is required, no subquadratic algorithms are known.

Furthermore, it is generally undesirable for the running time of an algorithm to scale poorly with the numerical values of the input.
In fact for many fundamental problems (Knapsack, All-Pairs Shortest Paths, 3-SUM), having small entries makes the problems much easier.
For example, there is a simple pseudo-polynomial time dynamic programming algorithm for Knapsack, while designing a polynomial time algorithm for Knapsack is NP-complete.\footnote{An algorithm runs in pseudo-polynomial time if its running time is polynomial in the numerical value of the input. A polynomial time algorithm must be polynomial in the length of the input.}
Therefore, in this work we study algorithms for Attention that scale polynomially with the \emph{representation length} of the entries.
Equivalently, the algorithm should scale polylogarithmically with the entry size $B$.

Currently, the only known algorithms for Attention beyond the standard $O(n^2 d)$ algorithm is Alman and Song's algorithm \cite{alman2023fastattentionboundedentries} which scales \emph{exponentially} with the entry size $B$.
Following the terminology of pseudo-polynomial time, we will call an algorithm that is subquadratic but scaling polynomially (or worse) with the numerical value of the input pseudo-subquadratic.
We call an algorithm that is subquadratic and scales logarithmically with the numerical value of the inputs (non-pseudo-)subquadratic, or simply subquadratic.
Following from our above discussion, the question of whether (non-pseudo-)subquadratic algorithms for Attention exist remains open.\footnote{Similarly, while there are pseudo-subcubic algorithms for APSP (e.g., \cite{DBLP:conf/focs/ShoshanZ99,DBLP:journals/jacm/Zwick02}), there is no truly subcubic ($O(n^{3-c})$ for some $c > 0$) algorithm.}
Even if $d = O(1)$, there is a tantalizing gap between the $O(n^2)$ upper bound and the $\Omega(n)$ lower bound.

\begin{center}
{\em Question 2: Is there a truly (non-pseudo-)subquadratic algorithm for Attention?\footnote{An algorithm runs in truly subquadratic time if it runs in $O(n^{2 - c})$ time for some $c > 0$}}
\end{center}

In our work, we resolve this question for almost all regimes of head dimension $d$, including the setting $d = \Theta(\log n)$ previously studied \cite{alman2023fastattentionboundedentries}.
Our first result makes progress towards resolving this question by developing a fast algorithm whenever the head dimension is small.
In particular, we show that there is an algorithm for Attention in truly sub-quadratic time for constant $d$.\footnote{We use $\tO{\cdot}$ notation to suppress polylogarithmic factors.} 

\begin{restatable}{theorem}{ConstantDAttentionAlg}
    \label{m-thm:const-d-att-alg}
    Let $d = O(1)$.
    There is an algorithm that computes $\attc(n, d, \entryBound, \varepsilon)$ in $\tO{n^{2 - 1/d} \cdot \polylog(\entryBound/\varepsilon)}$ time.
\end{restatable}

The result also generalizes to the case where the matrices $Q, K$ have low rank. 

\begin{restatable}{theorem}{ConstantRankAttentionAlg}
    \label{thm:low-rank-attention-alg}
    Let $r = O(1)$.
    There is an $\bigtO{n d + n^{2 - 1/r} \cdot \polylog(\entryBound/\varepsilon)}$ time algorithm computing $\attc(n, d, \entryBound, \varepsilon)$ where $r = \min(\rank(Q), \rank(K))$.
\end{restatable}

We complement this algorithm with a subquadratic algorithm for Attention Gradient Computation similar to \cite{alman2024fine}.
In the training process, gradient descent tunes the weight matrices $W_{Q}, W_{K}, W_{V}$ according to the input data.
When $d = O(1)$, we give a subquadratic running time for computing the Attention gradient.
In particular, we obtain a truly (non-pseudo-)subquadratic algorithm for the full LLM training process in this regime.

\begin{theorem}[Informal \Cref{thm:backwards-pass-alg}]
    \label{m-thm:backwards-pass-alg}
    The Attention gradient can be computed with $O(d)$ calls to $\attc(n, d, B, \varepsilon/\Theta(ndB^{3}))$ with $O(nd^2)$ overhead.
    In particular, if $d = O(1)$ the Attention gradient can be computed in $\tO{n^{2 - 1/d} \polylog(B/\varepsilon)}$ time.
\end{theorem}

We give a generic reduction from gradient computation to attention computation.
Thus, to obtain fast LLM training, it suffices to efficiently compute Attention.

When $d = \omega(1)$, the above algorithms requires $n^{2 - o(1)}$ time.
Is there a truly subquadratic algorithm for super-constant $d$?
\cite{alman2023fastattentionboundedentries} show that $n^{2 - o(1)}$ time is necessary when $d = \Omega(\log n)$  under the Strong Exponential Time Hypothesis ($\SETH$). 
Under the same hardness assumption we provide a much stronger lower bound and 
show that Attention is hard even when $d = 2^{\Omega(\log^* n)}$. 

\begin{theorem}[Informal \Cref{thm:super-const-d-lb}]
    \label{m-thm:super-const-d-lb}
    Under $\SETH$, $\attc(n, d, \entryBound, \varepsilon)$ requires $n^{2 - o(1)}$ time for $d = 2^{\Omega(\log^* n)}$ and $B = \poly(n)$.
\end{theorem}

It suffices to consider instances with polynomial entry size $B = \poly(n)$ since any (non-pseudo-) subquadratic algorithm must handle such instances in subquadratic time.
Formally, we show that any fast algorithm for $\attc(n, d, B, \varepsilon)$ implies a fast algorithm for (Bichromatic) Maximum Inner Product ($\mip$) on $d$-dimensional vectors with integer entries.
The (Bichromatic) $\mip$ problem asks an algorithm given two sets of vectors $A, B \subseteq \Z^{d}$ to compute $\max_{a \in A, b \in B} a \cdot b$.
Under $\SETH$, this requires $n^{2 - o(1)}$ time whenever $d = 2^{\Omega(\log^* n)}$ \cite{chen2018maxip}.
Furthermore, the best known algorithms for $\mip$ run in $n^{2 - \Theta(1/d)}$ time \cite{Yao1982, 1991Euclidean, Mat92PartitionTrees} so that any algorithm improving significantly over \Cref{m-thm:const-d-att-alg} must improve upon the best known algorithms for $\mip$.
\cite{chen2018maxip} conjectures that no such algorithm exists under \SETH.

{\bf Stronger Lower Bounds for Large Head Dimension.}
The head dimension $d$ can often be relatively large with respect to the context length $n$ (in some cases e.g. \cite{DBLP:conf/nips/VaswaniSPUJGKP17}, the head dimension $d$ can even be larger than the context length $n$).
In these settings, a large gap remains between the standard algorithm requiring $O(n^2 d)$ time and the known $n^{2 - o(1)}$ lower bound.
Our work addresses this gap and shows that the standard algorithm is conditionally optimal.

Our conditional lower bound depends on a natural generalization of a popular hypothesis.
The Orthogonal Vectors ($\OV$) problem is among the most well studied problems in fine-grained complexity.
In the $\OV$ problem, an algorithm is given two sets of $n$ vectors $A, B \subseteq \SET{0, 1}^{d}$ and is asked to determine if there exists an orthogonal pair $a \in A, b \in B$ such that $a \cdot b = 0$.
The naive algorithm for this problem requires $O(n^2 d)$ time and the current best algorithm for $\OV$ achieves truly subquadratic time only for $d = O(\log n)$ \cite{DBLP:conf/soda/AbboudWY15, DBLP:conf/soda/ChanW16}. A central hypothesis (known as the $\OV$ Hypothesis) in fine-grained complexity states that there is no $n^{2 - o(1)}$ algorithm for $\OV$ whenever $d = \omega(\log n)$, and the $\OV$ Hypothesis is known to hold under $\SETH$ \cite{williams2004ovc}.

If $d = \poly(n)$, one can compute $a \cdot b$ for all pairs $a \in A, b \in B$ using a matrix product between an $n \times d$ matrix containing the vectors in $A$ as rows and a $d \times n$ matrix containing the vectors of $B$ as columns.
The above algorithm requires $O(\TMUL(n, d, n))$ time, where $\TMUL(a, b, c)$ is the time complexity for multiplying an $a \times b$ matrix with a $b \times c$ matrix.
The High-Dimensional $\OV$ Hypothesis introduced by \cite{DBLP:conf/icalp/DalirrooyfardK21} hypothesized that when $d = n$, any algorithm computing $\OV$ requires $\TMUL(n, n, n)^{1-o(1)} = n^{\omega - o(1)}$ time, where $\omega < 2.3714$ denotes the square matrix multiplication exponent \cite{ DBLP:conf/soda/AlmanDWXXZ25}. We consider a generalization of their hypothesis: the $\TMUL(n, d, n)^{1-o(1)}$ running time is required for any $d = \poly(n)$. We call it the Generalized High-Dimensional $\OV$ Hypothesis.

Under this hypothesis, we show that the standard algorithm for computing Attention is optimal.
Note that using fast matrix multiplication, one can easily obtain an algorithm for Attention using $O(\TMUL(n, d, n))$ time.

\begin{theorem}[Informal \Cref{thm:poly-d-lb}]
    \label{m-thm:poly-d-lb}
    Under the Generalized High-Dimensional $\OV$ Hypothesis, $\attc(n, d, B, \varepsilon)$ requires $\TMUL(n, d, n)^{1 - o(1)}$ time for $d = \poly(n)$.
\end{theorem}

\Cref{tbl:intro-results-tbl} summarizes our results.
In particular, we tightly characterize the complexity of Attention (up to sub-polynomial factors) when $B = \poly(n)$ for all regimes of $d$ except $1 \ll d \ll 2^{\Theta(\log^* n)}$.
Within this regime, our running time matches the best known algorithms for $\mip$ \cite{Yao1982, 1991Euclidean, Mat92PartitionTrees},
and as mentioned earlier, significant improvements over our algorithm will imply improvements over the current best known algorithms for $\mip$ which will be a breakthrough. 

\begin{table}[h!]
    \centering
    \begin{tabular}{|| c | c | c | c ||}
         \hline
         $d$ & Upper Bound & Lower Bound & Previous Result \\
         \hline
         \hline
         $O(1)$ & $n^{2 - 1/d}$ (\Cref{m-thm:const-d-att-alg}) & $\Omega(n)$ & $O(n^2)$ \\
         $2^{\Theta(\log^* n)}$ & $n^{2}$ & $n^{2 - o(1)}$ (\Cref{m-thm:super-const-d-lb}) & $\Omega(nd)$ \\
         $\Theta(\log n)$ & $n^{2}$ & $n^{2 - o(1)}$ (\Cref{thm:ov-lb-constant}) & $n^{2 - o(1)}$ (\cite{alman2023fastattentionboundedentries})* \\
         $\poly(n)$ & $\TMUL(n, d, n)$ & $\TMUL(n, d, n)^{1 - o(1)}$ (\Cref{m-thm:poly-d-lb}) & $n^{2 - o(1)}$ (\cite{alman2023fastattentionboundedentries}) \\
         \hline
    \end{tabular}
    \caption{Summary of known results when $B = \poly(n)$ and $\varepsilon = 1/\poly(n)$. 
    Sub-polynomial dependencies are suppressed for simplicity.
    *For $d = \Theta(\log n)$, the previous lower bound holds when $B = \Omega(\sqrt{\log n})$ while ours holds even when $B \geq \log 2$.
    }
    \label{tbl:intro-results-tbl}
\end{table}

\subsection{Technical Overview}

In this section, we give a high level overview of our algorithm.
For simplicity, we focus on the $d = 1$ case  in this overview. 
Given inputs $q, k, v \in \R^{n}$, our goal is to compute $o_{i} = \sum_{j} p_{i, j} v_{j}$ for all $i$ where $p_{i, j} \propto \exp(q_{i} k_{j})$ are probabilities in the softmax distribution.

First, we show that for each $i$ it suffices to only consider near maximal $q_{i} k_{j}$.
Assume without loss of generality that $q_i > 0$.
Let $k_{\max} = \max_{j} k_{j}$. 
Define $j$ to be {\em irrelevant} (with respect to $q_{i}$) if $q_{i} k_{j} \ll q_{i} k_{\max} - \log(n/\varepsilon)$ and {\em relevant} otherwise. Therefore, $\exp(q_ik_j) \leq \frac{\varepsilon \exp(q_ik_{\max})}{n}$.
Summing over all irrelevant $j$, $p_{i, j} \ll \varepsilon$. 
Since discarding such $j$ does not significantly change the value of $o_{i}$, we consider only relevant $j$ for the remainder of the overview.

\begin{figure}[h!]
   \centering
   \usetikzlibrary{matrix}
\usetikzlibrary{math}

\begin{tikzpicture}[>=stealth]
    \begin{scope}[shift={(0,0)}]
        \tikzmath{\base=1; \width=3; \baser=\base+\width; \baseg=\baser+\width; \baseb=\baseg+\width; \start=\base-1; \endl=\baseb+1; \textx=\base-2;}
        
        \fill[red!30] plot[smooth, samples=100, domain=1:3.8] (\base, 0) -- (\base, 2) -- (\baser, 2) -- (\baser, 0) -- cycle;
        \draw (\base, 0) -- (\base, 2);
        
        \fill[yellow!30] plot[smooth, samples=100, domain=1:3.8] (\baser, 0) -- (\baser, 2) -- (\baseg, 2) -- (\baseg, 0) -- cycle;
        \draw (\baser, 0) -- (\baser, 2);
        
        \fill[green!30] plot[smooth, samples=100, domain=1:3.8] (\baseg, 0) -- (\baseg, 2) -- (\baseb, 2) -- (\baseb, 0) -- cycle;
        \draw (\baseg, 0) -- (\baseg, 2);
        
        \fill[blue!30] plot[smooth, samples=100, domain=1:3.8] (\baseb, 0) -- (\baseb, 2) -- (\endl, 2) -- (\endl, 0) -- cycle;
        \draw (\baseb, 0) -- (\baseb, 2);
        
        \draw[->] (\start, 1) -- (\endl, 1);
    
        \node[above=3pt] at (\base, 2) {Relevant Indices $\rightarrow$};

        \tikzmath{\x1=0.5; \x2=1.8; \x3=2.6; \x4=4.5; \x5=6.1; \x6=7.4; \x7=8.2; \x8=9.4;}
        
        \node[circle, draw, fill, scale=0.5] (k0) at (\x1, 1) {};
        \node[below=5pt] at (k0) {$k_{1}$};
        \node[below=30pt] at (k0) {$6$};
        \node[below=43pt] at (k0) {\x1};
        \node[below=56pt] at (k0) {$-\infty$};

        \node[circle, draw, fill, scale=0.5] (k1) at (\x2, 1) {};
        \node[below=5pt] at (k1) {$k_{2}$};
        \node[below=30pt] at (k1) {$1$};
        \node[below=43pt] at (k1) {\x2};
        \node[below=56pt] at (k1) {$4$};
    
        \node[circle, draw, fill, scale=0.5] (k2) at (\x3, 1) {};
        \node[below=5pt] at (k2) {$k_{3}$};
        \node[below=30pt] at (k2) {$3$};
        \node[below=43pt] at (k2) {\x3};
        \node[below=56pt] at (k2) {$4$};
    
        \node[circle, draw, fill, scale=0.5] (k3) at (\x4, 1) {};
        \node[below=5pt] at (k3) {$k_{4}$};
        \node[below=30pt] at (k3) {$8$};
        \node[below=43pt] at (k3) {\x4};
        \node[below=56pt] at (k3) {$7$};
    
        \node[circle, draw, fill, scale=0.5] (k4) at (\x5, 1) {};
        \node[below=5pt] at (k4) {$k_{5}$};
        \node[below=30pt] at (k4) {$5$};
        \node[below=43pt] at (k4) {\x5};
        \node[below=56pt] at (k4) {$7$};
    
        \node[circle, draw, fill, scale=0.5] (k5) at (\x6, 1) {};
        \node[below=5pt] at (k5) {$k_{6}$};
        \node[below=30pt] at (k5) {$7$};
        \node[below=43pt] at (k5) {\x6};
        \node[below=56pt] at (k5) {$10$};
    
        \node[circle, draw, fill, scale=0.5] (k6) at (\x7, 1) {};
        \node[below=5pt] at (k6) {$k_{7}$};
        \node[below=30pt] at (k6) {$4$};
        \node[below=43pt] at (k6) {\x7};
        \node[below=56pt] at (k6) {$10$};
    
        \node[circle, draw, fill, scale=0.5] (k7) at (\x8, 1) {};
        \node[below=5pt] at (k7) {$k_{8}$};
        \node[below=30pt] at (k7) {$2$};
        \node[below=43pt] at (k7) {\x8};
        \node[below=56pt] at (k7) {$10$};
    
        \node[below=30pt] at (\textx, 1) {Value ($v$)};
        \node[below=43pt] at (\textx, 1) {Key ($k$)};
        \node[below=54pt] at (\textx, 1) {Rounding ($\overline{k}$)};
    \end{scope}
\end{tikzpicture}
   \caption{Rounding based algorithm for $1$-dimensional Attention illustrated for $q_i = 1$. 
   Each point is placed at $k_j$ and has value $v_{j}$. 
   Points (e.g. $k_1$) such that $q_i k_j < q_i k_{\max} - \log(n/\varepsilon)$ are irrelevant and discarded (in this example $q_i k_{\max} - \log(n/\varepsilon) = 1$).
   Relevant points with similar $k_j$ (e.g. $\{k_{2}, k_{3}\}$ or $\{k_{6}, k_{7}, k_{8}\}$) are grouped together and assigned the same (rounded) key $\overline{k}$.
   The width of each region is $\log(1 + \varepsilon)$ (in this example $\log(1 + \varepsilon) = 3$).
   The algorithm outputs $\sum \overline{p}_{j} v_{j}$ where $\overline{p}_{j} \propto \exp(\overline{k}_{j})$.}
   \label{fig:1-d-attention}
\end{figure}
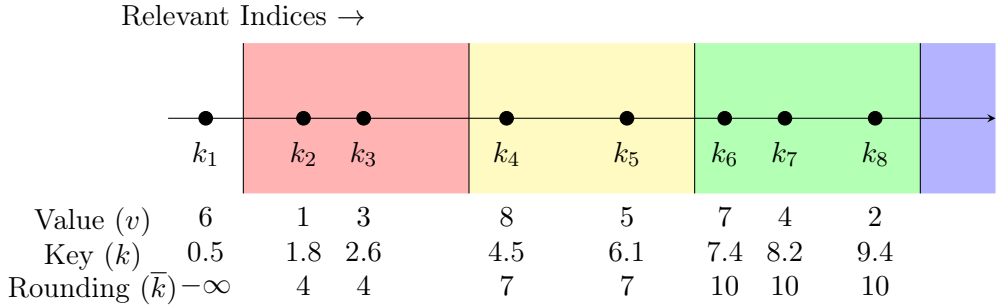

Combining this observation with a simple rounding scheme, we already obtain a modest improvement over known algorithms for Attention \cite{alman2023fastattentionboundedentries}. 
We illustrate this for the $d=1$ case.
Consider relevant $k_{j}$ such that $q_{i} k_{j} \geq q_{i} k_{\max} - \log(n/\varepsilon)$.
If we round such $k_{j}$ to $\overline{k}_{j}$ such that $q_{i} k_{j} \leq q_{i} \overline{k}_{j} \leq q_{i} k_{j} + \log(1 + \varepsilon)$, then $e^{q_{i} \overline{k}_{j}}$ is a $(1+\varepsilon)$-multiplicative approximation of $e^{q_{i} k_{j}}$.
This gives us estimates $\overline{p}_{i, j}$ that are good multiplicative approximations of $p_{i, j}$.
The algorithm then outputs $\hat{o}_{i} = \sum_{j} \overline{p}_{i, j} v_{j}$ which is a $(1+\varepsilon)$-multiplicative approximation of $o_{i}$.
Since $o_{i} = O(B)$, this gives a $\varepsilon B$-additive approximation of $o_{i}$.
To compute $\hat{o}_{i}$, we need to compute the number of $k_{j}$ in the $\frac{\log(n/\varepsilon)}{\log(1 + \varepsilon)} = O((1/\varepsilon) \log(n/\varepsilon))$ intervals for each $q_{i}$ since for all $k_j$s that fall in the same interval will have the same rounded value $\overline{k_j}$.
We can preprocess $k_{j}$ in $\tO{n}$ time to support such interval counting queries in $O(\log n)$ time. 
Since the total number of queries is $\tO{n/\varepsilon}$, we can compute $\attc(n, d, B, \varepsilon B)$ in time $\tO{n/\varepsilon}$, or $\attc(n, d, B, \varepsilon)$ in time $\tO{nB/\varepsilon}$ by rescaling $\varepsilon$ by $B$. 
\Cref{fig:1-d-attention} illustrates the rounding scheme.

The above rounding method gives a polynomial dependence on the entry bound $B$, and is only subquadratic when $B < o(n)$.
Although this already impoves on \cite{alman2023fastattentionboundedentries}'s algorithm (which exhibits exponential dependence on $B$, and thus only worked for values of $B = o(\sqrt{\log{n}})$), we would like a truly  subquadratic algorithm for all polynomial $B$.
To do this, we leverage the powerful polynomial method in algorithm design (see e.g. \cite{Williams18APSP, AbboudWY15}) that was also used in \cite{alman2023fastattentionboundedentries}. 

A natural attempt to utilize the polynomial method is to approximate $e^{x}$ with a polynomial.
As a simple case, by approximating $e^{x} \sim 1 + x$ we can compute $\exp(QK^T)V \sim \onevec \onevec^{\transpose} V + QK^TV$ efficiently. 
However, it is well known that $e^{x}$ can only be approximated well by polynomials with degree $p$ when $|x| \leq p$ \cite{aggarwal2022optimal}.
For a rank $d = O(1)$ matrix $QK^{\transpose}$, $\exp(QK^{\transpose})$ can be approximated with a rank $2^{O(B^2)}$ matrix.
Using this observation (as in \cite{alman2023fastattentionboundedentries}) one can obtain a subquadratic algorithm by assuming $B = o(\sqrt{\log n})$, but we would like to design fast algorithms in more general settings.

Let us see how we can handle 1-dimensional Attention.
For $x = O(\log(n/\varepsilon))$, there is a low-degree polynomial $P$ such that $|P(x) - \exp(x)| < \varepsilon \exp(x)$ \cite{alman2023fastattentionboundedentries}.
We then approximate $p$ with $\hat{p}$ by approximating $\exp$ with $P$.
Formally, our algorithm outputs $\hat{o}_{i} = \sum_{j} \hat{p}_{i, j} v_{j}$ where $\hat{p}_{i, j} \propto P(q_{i} k_{j})$.
Here, we crucially use the observation that irrelevant indices can be discarded.
Setting $c_i = \max_{j} q_i k_{j} - O(\log (n/\varepsilon))$ we can approximate $p_{i, j} \propto \exp(q_{i} k_{j}) \propto \exp(q_{i} k_{j} - c_{i})$.
For relevant indices $q_{i}k_{j} - c_{i} = O(\log(n/\varepsilon))$ so that $\hat{p}_{i, j} \propto P(q_{i} k_{j} - c_{i})$ is a good approximation of $p_{i, j}$.

We now describe how to compute $\hat{o}_{i}$ efficiently. 
In particular, we describe how to compute $\sum_{j} P(q_{i} k_{j} - c_{i}) v_{j}$ over relevant $j$ noting that $\hat{o}_{i}$ easily follows from computing this quantity twice (once with $v$ and once with $v$ replaced by $\onevec$ for normalization).
Let $P(x) = \sum_{\ell} m_{\ell} x^{\ell}$ and fix a monomial $x^{\ell}$.
It suffices to compute each monomial separately as $\sum_{j} P(q_{i} k_{j} - c_{i}) v_{j} = \sum_{j, \ell} m_{\ell} (q_{i} k_{j} - c_{i})^{\ell} v_{j}$.
Then for a fixed $\ell$ we compute
\begin{equation*}
    f(i, \ell) := \sum_{j} (q_{i} k_{j} - c_{i})^{\ell} v_{j} = \sum_{j} v_{j} \sum_{b = 0}^{\ell} \binom{\ell}{b} q_{i}^{b} k_{j}^{b} (- c_{i})^{\ell - b} = \sum_{b = 0}^{\ell} \binom{\ell}{b} q_{i}^{b} (- c_{i})^{\ell - b} \sum_{j} k_{j}^{b} v_{j} \text{.}
\end{equation*}
We can preprocess a simple data structure in $\tO{n}$ time for all $b$ that answers queries $\phi(i, b) = \sum_{\relevant~j} k_{j}^{b} v_{j}$ in near constant time, since relevant indices lie in a continuous interval.
Then, for each $i$, since $P$ is of low degree, we compute $\sum_{\ell} m_{\ell} f(i, \ell)$ in $\tO{1}$ time, thus obtaining a near-linear time algorithm to approximate Attention.

{\bf Generalizing to Higher Dimensions.}
What happens when we try to generalize this algorithm to higher dimensions?
In one dimension, we knew that for each $i$, the set of relevant $j$ included all $j$ where $q_{i} k_{j} \gg q_{i} k_{\max} - \log(n/\varepsilon)$.
In higher dimensions, our goal is similarly to compute a set of relevant indices $j$ relative to each $Q_{i}$ such that (1) discarding irrelevant indices outside this range does not significantly affect the additive error of our estimate and (2) the range of $Q_{i} \cdot K_{j}$ is now sufficiently restricted so that we can use a low-degree polynomial to approximate $\exp(Q_{i} \cdot K_{j})$.

In one dimension, the set of all relevant $j$ consists exactly of the set of sufficiently large $k_{j}$ (either in the positive or negative direction).
A simple interval searching data structure can support the necessary queries.
In $d > 1$ dimensions, each row of $Q, K$ (denoted $Q_{i}, K_{j}$) is now a $d$-dimensional vector.
Even in 2 dimensions, different $K_{j}$ may be larger with respect to different $Q_{i}$. 
Sorting all $K_{j}$ with respect to each $Q_{i}$ already requires $n^{2}$ time.
Instead, the key observation is that the set of relevant $j$ with respect to $Q_{i}$ is exactly the set of $K_{j}$ contained in the half-space
\begin{equation*}
    \SET{x \in \R^d :  Q_{i} \cdot x \geq \max_{j} Q_{i} \cdot K_{j} - \log(n/\varepsilon)}.
\end{equation*}
This can be handled with a simplex range-searching data structure by Matou{\v{s}}ek \cite{Mat92PartitionTrees}.
In particular, we can initialize the data structure using points $\SET{K_{j}}$ so that for each $Q_{i}$ we can query the data structure for the appropriate half-space.
Matou{\v{s}}ek's data structure supports queries in $\tO{n^{1 - 1/d}}$ time and computes the sum of the weights assigned to all points in the half-space.
In high dimensions, the application of the polynomial method becomes more challenging, and we carefully precompute weights assigned to each point.
Using appropriate queries to the data structure over all $i$, our algorithm requires $\tO{n^{2 - 1/d}}$ time.
\Cref{fig:constd-attention} illustrates the algorithm.

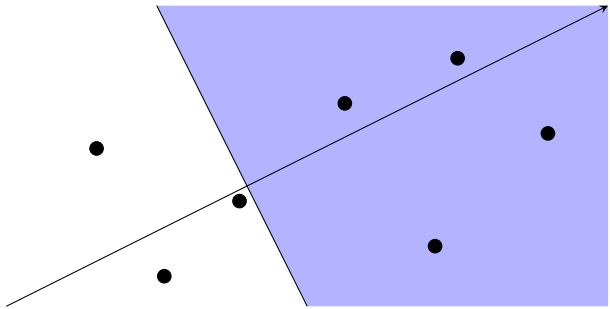
\begin{figure}[h!]
   \centering
   \usetikzlibrary{matrix}

\begin{tikzpicture}[>=stealth]
    \fill[blue!30] plot[smooth, samples=100, domain=1:3.8] (2, 4) -- (8, 4) -- (8, 0) -- (4, 0) -- cycle;
    
    \draw[->] (0, 0) -- (8, 4);
    \draw (2, 4) -- (4, 0);
    
    \node[circle, draw, fill, scale=0.5] (k1) at (7.2, 2.3) {};
    \node[circle, draw, fill, scale=0.5] (k2) at (6, 3.3) {};
    \node[circle, draw, fill, scale=0.5] (k3) at (1.2, 2.1) {};
    \node[circle, draw, fill, scale=0.5] (k4) at (4.5, 2.7) {};
    \node[circle, draw, fill, scale=0.5] (k5) at (3.1, 1.4) {};
    \node[circle, draw, fill, scale=0.5] (k6) at (5.7, 0.8) {};
    \node[circle, draw, fill, scale=0.5] (k7) at (2.1, 0.4) {};
\end{tikzpicture}
   \caption{Polynomial method algorithm for $d$-dimensional Attention illustrated for $q_{i} = (2, 1)$. Relevant points are in the shaded blue region. Irrelevant points are in the white region. Weights are omitted for clarity.}
   \label{fig:constd-attention}
\end{figure}

{\bf Generalizing to Low Rank Matrices.}
To generalize the algorithm for low-rank matrices $Q, K$ with rank $r$, we may decompose $Q = U_{Q} V_{Q}^{\transpose}, K = U_{K} V_{K}^{\transpose}$ where $U_{Q}, V_{Q}, U_{K},V_{K}$ are $n \times r$ matrices.
Then, we obtain \Cref{thm:low-rank-attention-alg} by applying \Cref{m-thm:const-d-att-alg} to $Q' = U_{Q}$ and $K'^{\transpose} = V_{Q}^{\transpose} U_{K} V_{K}^{\transpose}$ which may be computed in $O(nr)$ time.

{\bf Attention Gradient Computation.}
We revise the attention gradient computation formula, building on prior work \cite{alman2024fine}, to express the gradient as a composition of attention computations plus efficient matrix computations. 
By carefully selecting the order of the matrix operations, this formula can be evaluated with just attention computations and additional matrix operations taking $O(nd^2)$ time.

{\bf Outline.}
We give our algorithm in \Cref{sec:att-alg}.
The reduction from gradient computation to Attention computation is given in \Cref{sec:backwards-pass}.
Our lower bounds are presented in \Cref{sec:hardness}.

\subsection{Related Work}

Attention is a key bottleneck of the transformer architecture \cite{DBLP:conf/nips/VaswaniSPUJGKP17} which has become the predominant architecture in applications such as natural language processing.

{\bf Approximate Attention Computation.}
In an orthogonal line of work, many approximate notions of Attention have been studied to reduce its compute constraints with the goal of computing an approximation in linear time \cite{DBLP:conf/nips/BrownMRSKDNSSAA20,   DBLP:conf/icml/KatharopoulosV020, DBLP:conf/iclr/KitaevKL20,  DBLP:conf/nips/ZaheerGDAAOPRWY20, DBLP:conf/nips/ChenDWSRR21, DBLP:conf/iclr/ChoromanskiLDSG21, gao2023llmtensor, panigrahi2023trainable, malladi2023fine, song2024solving}.
Several works obtain provable guarantees as well as practical improvements, such as KDEFormer \cite{zandieh2023kdeformer}, HyperAttention \cite{hanhyperattention}, and PolySketchFormer \cite{kachampolysketchformer}.
However, these works only obtain theoretical guarantees with respect to matrix norms such as operator norm rather than any guarantee on the correctness of each entry.

Our work follows \cite{alman2023fastattentionboundedentries, alman2024fine} and focuses on the setting of computing Attention as exactly as possible (up to inverse polynomial precision in the output) and therefore obtains significantly stronger approximation guarantees.
Indeed, our lower bounds show that linear time approximations do not obtain such strong approximation guarantees.
In the low dimension regime $d = o(\log n)$, the Fast Multipole Method gives fast algorithms for the intimately related Gaussian Kernel Density Estimation (KDE) problem \cite{AlmanGuan24}.
However, these algorithms do not apply in our regime of polynomial entries.
In particular, using the standard reduction from Attention to Gaussian KDE\footnote{Map $x \mapsto (x, 0, r_{x})$ and $y \mapsto (y, r_{y}, 0)$ for appropriate $r_{x}, r_{y}$ so that $\norm{x - y}^{2} = R - 2 x \cdot y$ for some constant $R$.} the error produced by the known KDE algorithms is amplified so that only Attention with subpolynomial entries $B = 2^{o(\log n)}$ can be computed efficiently, even with constant dimension $d = O(1)$.

{\bf Attention with MLP Units.}
Many works have studied the expressive power of Transformers \cite{DBLP:conf/nips/SanfordHT23, Sanford0T24, SanfordFHTKHPM24, sanfordDepthWidth} for classical algorithmic problems.
In an independent work \cite{AlmanYu25} show that an Attention unit with input and output MLP Layers can compute $\OV$ and (Monochromatic) $\mip$.
While the constructions are essentially identical, we reduce (Bichromatic) $\mip$ to Attention, and thus obtain a strong conditional lower bound for $d = 2^{\Theta(\log^* n)}$ via \cite{chen2018maxip}.

Rather than allowing arbitrary inputs $Q, K, V \in \R^{n \times d}$, these works consider Attention with MLP Units: Given inputs $X \in \R^{n \times d_1}$ and $W_Q, W_K, W_V \in \R^{d_1 \times d}$, compute $Q = X W_Q, K = X W_K, V = X W_V$ and then $\att(Q, K, V)$.
This preprocessing step requires only $O(nd^2)$ time and does not change the running time of our algorithm.
Via a simple modification (to either our construction or \cite{AlmanYu25}),\footnote{We describe how to obtain $Q, K$. 
Given sets of vectors $A, B \subset \R^{d}$, let $X \in \R^{n \times 2d}$ consist of $A$ in the first $d$ columns, $B$ in the next $d$ columns.
Let $W_Q = \begin{pmatrix}
    I_{d} & 0
\end{pmatrix}$ and $W_K = \begin{pmatrix}
    0 & I_d
\end{pmatrix}$.} 
it is possible to show that an Attention unit with MLP Units can compute (Bichromatic) $\mip$.
Our reductions from $\OV$ naturally hold for bichromatic instances as well.

{\bf Alternative Settings of Attention Computation.}
Attention has also been studied in several settings, including dynamic updates \cite{brand2023algorithm}, quantum algorithms \cite{gao2023fast}, differential privacy \cite{gao2023differentially}, and I/O complexity \cite{sy:24}.
Conditional lower bounds for Attention have been studied as well \cite{keles2023computational, alman2023fastattentionboundedentries, alman2024fine, AlmanYu25}.

\section{Preliminaries}
\label{sec:prelim}

Let $\log$ denote the natural log.
Let $[n] = \{1, 2, \dotsc, n\}$.
For a matrix $M \in \R^{n \times m}$, we denote its $(i, j)$-entry by $M_{i, j}$, its transpose $M^{\transpose}$, and its inverse $M^{-1}$.
Let $\norm{M}_\infty := \max_{i, j} |M_{i, j}|$ and $\exp(M)$ denote applying $e^{x}$ entry-wise to $M$.
Let $\zerovec$ and $\onevec$ denote the all zeros and all ones vectors.
For a vector $v \in \R^n$, $\diag(v)$ denotes the $n \times n$ diagonal matrix whose $(i,i)$-entry equals $v_i$.

{\bf Fine-grained Complexity Hypotheses.}
We establish new fine-grained lower bounds for the approximate attention computation problem $\attc(n, d, B, \varepsilon)$.
These lower bounds are conditional on some well-known fine-grained complexity hypotheses, which we introduce below.

The Strong Exponential Time Hypothesis ($\SETH$) was introduced by \cite{impagliazzo2001seth}.
They hypothesized that solving $k$-$\sat$ for $k \geq 3$ cannot be significantly improved beyond exhaustive search.

\begin{hypothesis}[Strong Exponential Time Hypothesis ($\SETH$)]\label[hypothesis]{def:seth}
     For every $\varepsilon > 0$, there is a positive integer $k \geq 3$ such that $k$-$\sat$ on formulas with $n$ variables cannot be solved in $O(2^{(1-\varepsilon)n})$ time, even by randomized algorithms.
\end{hypothesis}

$\SETH$ is a strengthening of the famous $\ptime \neq \np$ conjecture and has later been used to derive fine-grained lower bounds for many fundamental computational problems, from string edit distance \cite{BackursI18} to graph diameter \cite{RodittyW13}.
Our lower bounds under $\SETH$ will proceed via reduction to the Orthogonal Vectors ($\OV$) Problem and the $\mip$ Problem.

\begin{theorem}[\cite{williams2004ovc}]
    \label{thm:ov-lower-bound}
    Assuming $\SETH$, for any $\delta > 0$ there is a constant $C$ such that any randomized algorithm solving $\OV$ in dimension $d = C \log n$ with high probability requires $\Omega(n^{2 - \delta})$ time.
\end{theorem}

The $\mip$ problem asks to compute given sets of integer-valued vectors $A, B \in \Z^d$, $\max_{a \in A, b \in B} a \cdot b$. 
\cite{chen2018maxip} showed that computing $\mip$ requires $n^{2-o(1)}$ time even when $d = 2^{\Theta(\log^* n)}$.

\begin{restatable}[\cite{chen2018maxip}]{theorem}{ZMIPLB}
    \label{thm:z-mip-lower-bound}
    Assuming $\SETH$, for any $\delta > 0$ there is a constant $C$ such that any exact algorithm for $\mip$ in dimension $d = C^{\log^* n}$ with $O(\log n)$-bit entries requires $\Omega(n^{2 - \delta})$ time.
\end{restatable}

\section{Fast Attention for Constant Head Dimension}
\label{sec:att-alg}

We present our algorithms for computing Attention in truly subquadratic time for constant head dimension $d$ and polynomial entry size $\entryBound$.

\ConstantDAttentionAlg*

The algorithm naturally extends to the case when $d$ is large but the matrices are low dimensional.

\subsection{Warm-up: \texorpdfstring{$d = 1$}{d = 1}}
\label{app:warmup}

For simplicity, we begin with our algorithm for the $d = 1$ case and explain how to generalize to the constant head dimension case later.
Formally, we prove in this section the following result.

\begin{lemma}
    \label{lem:1-d-attention-alg-faster}
    There is an algorithm computing $\attc(n, 1, B, \varepsilon)$ in $\tO{n \cdot \polylog(\entryBound/\varepsilon)}$ time.
\end{lemma}
In the above, $\att(q, k, v)$ is defined by viewing vectors $q, k, v \in \R^{n \times 1}$ as matrices.
When $d = 1$, the input is given by vectors $q, k, v \in [-\entryBound, \entryBound]^{n}$.
In the output vector, we hope to compute the entries
\begin{equation*}
    o_{i} = \frac{\sum_{j} e^{q_i k_j} v_j}{\sum_{j} e^{q_i k_j}}
\end{equation*}
for all $i$.
Define the softmax probabilities 
\begin{equation*}
    p_{i, j} = \frac{e^{q_i k_j}}{\sum_{j'} e^{q_i k_{j'}}}
\end{equation*}
so that $o_{i} = \sum_{j} p_{i, j} v_j$.

We begin with an overview of our algorithm.
Without loss of generality, we assume $q_{i} \geq 0$ are non-negative for all $i$.
In particular, if we compute $\att(|q|, k, v)$ and $\att(|q|, -k, v)$, where $|q|$ is a vector where we take entrywise absolute value of $q$,  we can recover the entries of $\att(q, k, v)$ from the two outputs.
If $q_{i} \geq 0$, we read the output from $\att(|q|, k, v)$ and otherwise we read the output from $\att(|q|, -k, v)$.

Let $k_{\max}$ denote the maximum value of $k$ and $p_{\max}^{(i)} = \max_{j} p_{i, j}$ be the corresponding maximum probability for some fixed $i$.
First, we argue that we may ignore all indices where $k_j \ll k_{\max}$.
Since all of these indices have exponentially small $p_{i, j}$, ignoring these indices incurs only a small additive error to the output estimate $\hat{o}_{i}$.
Second, we argue that the remaining values of $k_{j}$ satisfy the property that $q_{i} k_{j}$ lie in a small range.
In particular, on this range, we use the low-degree polynomial $P$ from  \cite{aggarwal2022optimal} to give a low-error approximation of the exponential function.
Using this polynomial approximation, we instead compute
\begin{equation*}
    \hat{o}_{i} = \frac{\sum_{j} P(q_{i} k_{j} - c) v_{j}}{\sum_{j} P(q_{i} k_{j} - c)} \approx \frac{\sum_{j} e^{q_i k_j - c} v_{j}}{\sum_{j} e^{q_i k_j - c}} = \frac{\sum_{j} e^{q_i k_j} v_{j}}{\sum_{j} e^{q_i k_j}} = o_{i}
\end{equation*}
for some value $c$ that guarantees $q_{i} k_{j} - c$ lies in a bounded interval around $0$ for the remaining values $k_{j}$. 

Consider a monomial $m_{\ell} x^{\ell}$ of $P$.
Then $\sum_{j} (q_{i} k_{j} - c)^{\ell} = \sum_{b = 0}^{\ell} \binom{\ell}{b} (-c)^{\ell - b} q_{i}^{b} \sum_{j} k_{j}^{b}$.
This allows to pre-compute $\sum_{j} k_{j}^{b}$ for all exponents $b$ in a pre-processing phase, and then efficiently compute $\hat{o}_{i}$ using the pre-computed values.
We now describe the algorithm in more detail.

\paragraph{Step 1: Removing Irrelevant Keys.}

We argue that we can ignore irrelevant keys $k_j$ with only small additive error in the estimate.

\begin{definition}
    \label{def:relevant-keys}
    An index $j \in [n]$ is {\em relevant} with respect to $q_{i}$ if $q_{i} k_j \geq \max_{j} q_{i} k_{j} - \log(n/\varepsilon)$. 
    Otherwise, we say that $j$ is {\em irrelevant} with respect to $q_{i}$.
    When $q_{i}$ is clear from context, we simply say $j$ is relevant (or irrelevant).
\end{definition}
Since $q_{i} \geq 0$, by rearranging, note that for all irrelevant $j$, we have $q_{i} k_j - q_{i} k_{\max} \leq - \log(n/\varepsilon)$.
Then, we conclude
\begin{equation*}
    \frac{p_{i, j}}{p_{\max}^{(i)}} = e^{q_i (k_j - k_{\max})} \leq \frac{\varepsilon}{n}.
\end{equation*}
Summing over all such indices $j$,
\begin{equation*}
    \sum_{\irrelevant~j} p_{i, j} \leq \sum_{\irrelevant~j} p_{\max}^{(i)} \frac{\varepsilon}{n} \leq \varepsilon.
\end{equation*}
Thus, if we define 
\begin{equation*}
    p_{i, j}^{(r)} = \begin{cases}
        \frac{p_{i, j}}{\sum_{\relevant~j'} p_{i, j'}} & \textrm{$j$ is relevant,} \\
        0 & \otherwise,
    \end{cases}
\end{equation*} 
we can obtain the guarantees for all relevant $j$
\begin{equation*}
    p_{i, j} \leq p_{i, j}^{(r)} \leq \frac{p_{i, j}}{1 - \varepsilon}.
\end{equation*}
Then, define
\begin{equation*}
    o_i^{(r)} = \sum_{j} p_{i, j}^{(r)} v_j
\end{equation*}
so that
\begin{align*}
    \left| o_{i}^{(r)} - o_{i} \right| &\leq \left|\sum_{\relevant~j} (p_{i, j}^{(r)} - p_{i, j}) v_{j}\right| + \left|\sum_{\irrelevant~j} (p_{i, j}^{(r)} - p_{i, j}) v_{j}\right| \\
    &\leq \left|\sum_{\relevant~j} \frac{\varepsilon}{1 - \varepsilon} p_{i, j} v_{j}\right| + \left|\sum_{\irrelevant~j} p_{i, j} v_{j}\right| \\
    &\leq \frac{\varepsilon}{1 - \varepsilon} \entryBound + \varepsilon \entryBound \\
    &\leq 3 \varepsilon \entryBound
\end{align*}
where we assume $\varepsilon < \frac{1}{2}$.

\paragraph{Step 2: Polynomial Approximation of Exponential.}

We now show how we use polynomial approximations of $e^{x}$ to efficiently estimate attention.

Our goal is to approximate $o^{(r)}$:
\begin{equation*}
    o_{i}^{(r)} = \sum_{\relevant~j} p_{i, j}^{(r)} v_{j} = \frac{\sum_{\relevant~j} e^{q_{i} k_{j}} v_{j}}{\sum_{\relevant~j} e^{q_{i} k_{j}}} = \frac{\sum_{\relevant~j} e^{q_{i} k_{j} - c(q_{i})} v_{j}}{\sum_{\relevant~j} e^{q_{i} k_{j} - c(q_{i})}} 
\end{equation*}
where $c(q_i) = q_i \cdot k_{\max} - \log(n/\varepsilon)$.
In particular, we have $q_{i} k_{j} - c(q_{i}) \in \left[ 0, \log(n/\varepsilon) \right]$ for every relevant $j$.
The following result allows us to approximate the exponential function with polynomials on a bounded interval.

\begin{lemma}[\cite{aggarwal2022optimal, alman2023fastattentionboundedentries}]
    \label[lemma]{lem:poly-mult-approx}
    Let $\polyBound > 1$ and $\varepsilon < 0.1$.
    There is a polynomial $P: \R \rightarrow \R$ of degree $g = \bigTh{\max\left( \frac{\log(1/\varepsilon)}{\log(\log(1/\varepsilon)/\polyBound)}, \polyBound \right)}$ such that for all $x \in [-\polyBound, \polyBound]$, we have $
        |P(x) - \exp(x)| < \varepsilon \exp(x). $ 
    Moreover, its coefficients are rationals with $\poly(g)$-bit integer numerators and denominators and can be computed in $\poly(g)$-time.
\end{lemma}

On this interval, by \cref{lem:poly-mult-approx}, there is a polynomial $P$ of degree \[
g = \bigO{\max\left( \frac{\log(1/\varepsilon)}{\log(\log(1/\varepsilon)/\log(n/\varepsilon))}, \log(n/\varepsilon) \right)} = \bigO{\log(n/\varepsilon)}\]
such that $|P(x) - \exp(x)| \leq \varepsilon \exp(x)$ for all $x \in \left[ 0, \log(n/\varepsilon) \right]$.
Then, we define $\hat{p}_{i, j} = \frac{P(q_{i} k_{j} - c(q_{i}))}{\sum_{\relevant~j'} P(q_{i} k_{j'} - c(q_{i}))}$ for relevant $j$ and $\hat{p}_{i, j} = 0$ otherwise.
Next, define $\hat{o}_{i} = \sum_{j} \hat{p}_{i, j} v_{j}$.
First, we prove the desired approximation guarantee.
For all relevant $j$,
\begin{align*}
    \frac{1 - \varepsilon}{1 + \varepsilon} p_{i, j}^{(r)} \leq \hat{p}_{i, j} &\leq \frac{1 + \varepsilon}{1 - \varepsilon} p_{i, j}^{(r)}
\end{align*}
so that
\begin{align*}
    \left| \hat{o}_{i} - o_{i}^{(r)} \right| &\leq \entryBound \sum_{\relevant~j} \left|\hat{p}_{i, j} - p_{i, j}^{(r)}\right| \\
    &\leq \entryBound \sum_{\relevant~j} 4 \varepsilon p_{i, j}^{(r)} \leq 4 \varepsilon \entryBound. 
\end{align*}
Combined with our previous bound using triangle inequality, we get
\begin{equation}
    \label{eq:1-d-att-approx}
    \norm{\hat{o} - o}_{\infty} \leq \norm{\hat{o} - o^{(r)}}_{\infty} + \norm{o^{(r)} - o}_{\infty} \leq 7 \varepsilon \entryBound.
\end{equation}

Now, we describe how to compute $\hat{o}$ efficiently.
Consider a monomial $m_{\ell} x^{\ell}$ of $P$.
Then,
\begin{equation*}
    m_{\ell} (q_{i} k_{j} - c(q_{i}))^{\ell} = m_{\ell} \sum_{b = 0}^{\ell} \binom{\ell}{b} q_{i}^{b} k_{j}^{b} \left( - c(q_{i}) \right)^{\ell - b}
\end{equation*}
Summing over the indices $j$,
\begin{align*}
    \sum_{\relevant~j} m_{\ell} (q_{i} k_{j} - c(q_{i}))^{\ell} &= m_{\ell} \sum_{\relevant~j} \sum_{b = 0}^{\ell} \binom{\ell}{b} q_{i}^{b} k_{j}^{b} \left( - c(q_{i}) \right)^{\ell - b} \\
    &= m_{\ell} \sum_{b = 0}^{\ell} \binom{\ell}{b} q_{i}^{b} \left(-c(q_{i})\right)^{\ell - b} \sum_{\relevant~j} k_{j}^{b}
\end{align*}
Let $\phi(i, b) = \sum_{\relevant~j} k_{j}^{b}$ be the sum of $k_{j}^{b}$ for all $j$ relevant with respect to $q_{i}$.
In particular,
\begin{align*}
    \sum_{\relevant~j} P(q_{i} k_{j} - c(q_{i})) &= \sum_{\relevant~j} P(q_{i} k_{j} - c(q_{i})) \\
    &= \sum_{\relevant~j} \sum_{\ell} m_{\ell} \left( q_{i} k_{j} - c(q_{i}) \right)^{\ell} \\
    &= \sum_{\ell} m_{\ell} \sum_{b = 0}^{\ell} \binom{\ell}{b} q_{i}^{b} \left(-c(q_{i})\right)^{\ell - b} \phi(i, b).
\end{align*}
Following similar computations we obtain
\begin{equation}
    \label{eq:1-d-att-poly}
    \sum_{j} P(q_{i} k_{j} - c(q_{i})) v_{j} = \sum_{\ell} m_{\ell} \sum_{b = 0}^{\ell} \binom{\ell}{b} q_{i}^{b} \left(-c(q_{i})\right)^{\ell - b} \phi_{v}(i, b)
\end{equation}
where $\phi_{v}(i, b) = \sum_{\relevant~j} k_{j}^{b} v_{j}$.

The following lemmas show that we can compute $\hat{o}$ efficiently. 

\begin{lemma}
    \label[lemma]{lem:compute-phi-i-b}
    Let $b \geq 1$ and $k_{1} \geq k_{2} \geq \dotsc \geq k_{n}$.
    Let $q_{1}, \dotsc, q_{n}$ be arbitrary.
    Then, $\phi(i, b), \phi_{v}(i, b)$ can be computed for all $i$ in time $O(n \log n)$ time.
\end{lemma}

\begin{proof}
    Given $b$, we can compute $\sum_{j = 1}^{J} k_{j}^{b}$ for all $1 \leq J \leq n$ in $O(n)$ time.
    Then, for each $i$, we use binary search to find $J_{i}$, the maximum index $j$ where $k_j \geq \max_{j} k_{j} - \log(n/\varepsilon) / q_{i}$, i.e., $k_{j}$ is relevant with respect to $q_{i}$. Then we assign $\phi(i, b) = \sum_{j = 1}^{J_{i}} k_{j}^{b}$.
    Over all $i$, this takes $O(n \log n)$ time. We can
    $\phi_{v}(i, b)$  similarly.
\end{proof}

\begin{algorithm}

\SetKwInOut{Input}{Input}\SetKwInOut{Output}{Output}\SetKwInOut{Parameters}{Parameters}
\Input{vectors $q, k, v \in [-\entryBound, \entryBound]^{n}$.}
\Parameters{error parameter $\varepsilon$}
\Output{$\hat{o}$ satisfying $\norm{\hat{o} - \att(q, k, v)}_{\infty} \leq \varepsilon \entryBound$.}

Compute a polynomial $P(x) = \sum_{\ell} m_{\ell} x^{\ell}$ for range $[0, \log(n/\varepsilon)]$ using \Cref{lem:poly-mult-approx}.

Compute $k_{\max} \gets \max_{j} k_{j}$ and sort $\SET{k_{j}}$.

Compute $\phi(i, b), \phi_{v}(i, b)$ for all $1 \leq i \leq n, 1 \leq b \leq g$ using \Cref{lem:compute-phi-i-b}.

\For{$1 \leq i \leq n$}{
    Compute $\hat{o}_{i} \gets \frac{\sum_{\relevant~j} P(q_{i} k_{j} - c(q_{i})) v_{j}}{\sum_{\relevant~j} P(q_{i} k_{j} - c(q_{i}))}$ using \Cref{lem:compute-hat-o}.
}

\Return $\hat{o}$
\caption{$\vecAttF(q, k, v)$} 
\label{alg:1-d-attention-faster}

\end{algorithm}

\begin{lemma}
    \label[lemma]{lem:compute-hat-o}
    Let $P(x) = \sum_{\ell} m_{\ell} x^{\ell}$ be a degree $g$-polynomial with $\poly(g)$-bit coefficients.
    Given $q_{i}, \phi(i, b), \phi_{v}(i, b)$, there is an algorithm computing $\hat{o}_{i}$ in $\poly(g)$ time.
\end{lemma}

\begin{proof}
    We recall that
    \begin{equation*}
        \hat{o}_{i} = \sum_{j} \hat{p}_{i, j} v_{j} = \frac{\sum_{j} P(q_{i} k_{j} - c(q_{i})) v_{j}}{\sum_{j} P(q_{i} k_{j} - c(q_{i}))}.
    \end{equation*}
    From \Cref{eq:1-d-att-poly}, we note
    \begin{equation*}
        \sum_{j} P(q_{i} k_{j} - c(q_{i})) v_{j} = \sum_{\ell} m_{\ell} \sum_{b = 0}^{\ell} \binom{\ell}{b} q_{i}^{b} \left(-c(q_{i})\right)^{\ell - b} \phi_{v}(i, b)
    \end{equation*}
    so that given access to $\phi_{v}(i, b)$, we can compute the numerator in $\poly(g)$-time.
    Similarly, by accessing $\phi(i, b)$, we can compute the denominator as well.
\end{proof}

To conclude the proof of \Cref{lem:1-d-attention-alg-faster}, we apply \Cref{alg:1-d-attention-faster} with $\varepsilon' = \frac{\varepsilon}{7\entryBound}$ so we obtain $\varepsilon$-approximation under \Cref{eq:1-d-att-approx}.
In particular, the degree of the polynomial required is
\begin{equation*}
    g = \bigO{\log(n/\varepsilon')} = \bigO{\log(n\entryBound/\varepsilon)}.
\end{equation*}
Then, \Cref{alg:1-d-attention-faster} takes time $\tO{n\cdot \polylog(\entryBound/\varepsilon)}$.

\subsection{Algorithm for Attention with Constant Head Dimension}

We now give a formal description of our algorithm for constant head dimension and prove \Cref{m-thm:const-d-att-alg}.
For $d > 1$, we require an efficient data structure for the range searching problem.

\begin{definition}[Simplex Range Searching]
    \label[definition]{def:range-searching}
    Preprocess a weighted point set $\SET{(k_{i}, w_{i})}$ where $k_{i} \in \R^{d}$ and $w_{i} \in \R$ so that given any simplex query $\sigma$, the data structure returns $\sum_{k_{i} \in \sigma} w_i$.
\end{definition}

Matou{\v{s}}ek \cite{Mat92PartitionTrees} gives an efficient data structure for the simplex range searching problem.  In our work, we will only query the data structure with halfspaces $\sigma$, which are special case of simplex queries (one can imagine a simplex defined by the half-space and a sufficiently large bounding box that contains all input points).

\begin{theorem}[\cite{Mat92PartitionTrees}]
    \label[theorem]{thm:range-searching-ds}
    There is a data structure $\rsds$ for the Simplex Range Searching problem for $n$ input points in $d$-dimension with $O(n \log n)$ preprocessing and $\tO{n^{1 - 1/d}}$ query time.
\end{theorem}

Given this data structure, we now present our algorithm for arbitrary head dimension $d$.
Our inputs are $n \times d$ matrices $Q, K, V$ with entries in $[-\entryBound, \entryBound]$.
We rewrite $O_{i, t} = \sum_{j} p_{i, j} V_{j, t}$ where $p_{i, j} = \frac{\exp(Q_{i} \cdot K_{j})}{\sum_{j'} \exp(Q_{i} \cdot K_{j'})} \propto \exp(Q_{i} \cdot K_{j})$.

\paragraph{Step 1: Removing Irrelevant Keys.}
We begin by showing that removing irrelevant keys does not significantly alter the quality of the approximation.
Define for each $i \in [n]$ the maximum probability in the distribution $p_{i, j}$ as $p_{\max}^{(i)} = \max_{j} p_{i, j}$.
Let $s_{\max}^{(i)}$ denote the maximum integer $s$ such that the half-space
\begin{equation*}
    \SET{x \in \R^d :  Q_{i} \cdot x \geq s \log(1 + \varepsilon)}
\end{equation*}
contains at least one $K_{j}$ vector.
In particular, $s_{\max}^{(i)}$ is the largest integer satisfying $\max_{j} Q_{i} \cdot K_{j} \geq s_{\max}^{(i)} \log(1 + \varepsilon)$.
We now define relevant and irrelevant keys analogously to the $d = 1$ case.

\begin{definition}
    \label{def:irrelevant}
    Let $j \in [n]$ be {\em irrelevant} with respect to $Q_{i}$ if $Q_{i} \cdot K_{j} < s_{\max}^{(i)} \log(1 + \varepsilon) - \log(n/\varepsilon)$.
    Otherwise $j$ is {\em relevant} with respect to $Q_{i}$.
    When $Q_{i}$ is clear, we simply say $j$ is irrelevant or relevant.
\end{definition}

We argue that we can discard irrelevant indices.

\begin{lemma}
    Define $p_{i, j}^{(r)} = \frac{p_{i, j}}{\sum_{\textrm{relevant $j$}} p_{i, j}}$ if $j$ is relevant and $0$ otherwise for all $i, j \in [n]$.
    Let $O_{i, t}^{(r)} = \sum_{j} p_{i, j}^{(r)} V_{j, t}$ for all $i \in [n], t \in [d]$.
    Then $\left| O_{i, t}^{(r)} - O_{i, t} \right| \leq 3 \varepsilon \entryBound$.
\end{lemma}

\begin{proof}
    First, for all $i$ and $j$ irrelevant to $Q_{i}$, we have
    \begin{align*}
        \frac{p_{i, j}}{p_{\max}^{(i)}} = \frac{\exp(Q_{i} \cdot K_{j})}{\max_{j} \exp(Q_{i} \cdot K_{j})} \leq \frac{\exp(s_{\max}^{(i)} \log(1 + \varepsilon) - \log(n/\varepsilon))}{\exp(s_{\max}^{(i)} \log(1 + \varepsilon))} \leq \frac{\varepsilon}{n},
    \end{align*}
    so that
    \begin{equation}
        \sum_{\irrelevant~j} p_{i, j} \leq \sum_{\irrelevant~j}  p_{\max}^{(i)} \frac{\varepsilon}{n}  \leq \varepsilon.
    \end{equation}
    Thus for all relevant $j$, $p_{i, j} \leq p_{i, j}^{(r)} \leq \frac{p_{i, j}}{1 - \varepsilon}$.
    Following identical arguments as in the one-dimensional warm-up, we obtain the desired result.
\end{proof}

\paragraph{Step 2: Polynomial Approximation of Exponential.} 
Consider an entry $O_{i, t}$.
We now aim to approximate $O_{i, t}^{(r)}$.
Recall that
\begin{equation*}
    O_{i, t}^{(r)} = \sum_{j} p_{i, j}^{(r)} V_{j, t} = \frac{\sum_{\relevant~j} \exp(Q_{i} \cdot K_{j}) V_{j, t}}{\sum_{\relevant~j} \exp(Q_{i} \cdot K_{j})} = \frac{\sum_{\relevant~j} \exp(Q_{i} \cdot K_{j} - c(Q_{i})) V_{j, t}}{\sum_{\relevant~j} \exp(Q_{i} \cdot K_{j} - c(Q_{i}))}
\end{equation*}
where $c(Q_{i}) = s_{\max}^{(i)} \log(1 + \varepsilon) - \log(n/\varepsilon)$.
By the definition of $s_{\max}^{(i)}$, we have that for all relevant $j$, $Q_{i} \cdot K_{j} - C(Q_{i}) \in [0, \log(n/\varepsilon) + \log(1 + \varepsilon)]$.

We then invoke \Cref{lem:poly-mult-approx} to obtain a $g = \polylog(n/\varepsilon)$-degree polynomial $P$ such that for all $x \in [0, \log(n/\varepsilon) + \log(1 + \varepsilon)] \subset [0, 2 \log(n/\varepsilon)]$, $|P(x) - \exp(x)| \leq \varepsilon \exp(x)$.
Define for relevant $j$, $\hat{p}_{i, j} \propto P(Q_{i} \cdot K_{j} - c(Q_{i}))$ as an approximation of $p_{i, j}^{(r)} \propto \exp(Q_{i} \cdot K_{j} - c(Q_{i}))$.
For irrelevant $j$, set $\hat{p}_{i, j} = p_{i, j}^{(r)} = 0$.
Then, define $\hat{O}_{i, t} = \sum_{j} \hat{p}_{i, j} V_{j, t}$.

We claim $\hat{O}_{i, t}$ is a good approximation.

\begin{lemma}
    \label{lem:const-d-approx}
    $|\hat{O}_{i, t} - O_{i, t}| \leq 7 \varepsilon \entryBound$ for all $i \in [n], t \in [d]$.
\end{lemma}

This follows from identical arguments as in the 1-dimensional case.
Furthermore, we present an algorithm that computes $\hat{O}$ efficiently.
The key ingredient to the algorithm is the following data structure which again utilizes the range searching data structure of \cite{Mat92PartitionTrees}.

\begin{restatable}{lemma}{SmallDimensionDataStructure}
    \label{lem:small-dimension-data-structure}
    Given matrices $Q, K, V \in \R^{n \times d}$ there exist functions $\phi_{0}, \dotsc \phi_{d}$ such that any entry $\hat{O}_{i, t}$ can be computed with $g^{O(d)}$ queries to $\phi_{0}$ and $\phi_{t}$ and $g^{O(d)}$ additional time.
    
    Furthermore, for each $\phi_{t}$ with $0 \leq t \leq d$ there is a data structure with $\bigtO{g^{O(d)} n \log n}$ preprocessing and $\bigtO{g^{O(d)} n^{1 - 1/d} \log(B/\varepsilon)}$ query time.
\end{restatable}

\begin{proof}
    Recall that $\hat{O}_{i, t} = \frac{\sum_{\relevant~j}P(Q_{i} \cdot K_{j} - c(Q_{i})) V_{j, t}}{\sum_{\relevant~j}{P(Q_{i} \cdot K_{j} - c(Q_{i}))}}$ where $P$ is the polynomial of degree $g$ obtained from \Cref{lem:poly-mult-approx}.
    
    We begin with describing how to compute the numerator of $\hat{O}_{i, t}$.
    Suppose $P(x) = \sum_{\ell = 0}^{g} m_{\ell} x^{\ell}$.
    \begin{align*}
        & \sum_{\relevant~j} P(Q_{i} \cdot K_{j} - c(Q_{i})) V_{j, t}\\
        &= \sum_{\relevant~j} \sum_{\ell} m_{\ell} (Q_{i} \cdot K_{j} - c(Q_{i}))^{\ell} V_{j, t} \\
        &= \sum_{\ell} m_{\ell} \sum_{\relevant~j} \sum_{\ell_0+\ell_1 + \dotsc + \ell_{d} = \ell} \binom{\ell}{\ell_{0}, \ell_{1}, \dotsc, \ell_{d}} (-c(Q_{i}))^{\ell_{0}} \prod_{k = 1}^{d} \left( Q_{i, k} K_{j, k} \right)^{\ell_{k}} V_{j, t} \\
        &= \sum_{\ell} m_{\ell} \sum_{\ell_0+\ell_1 + \dotsc + \ell_{d} = \ell} \binom{\ell}{\ell_{0}, \ell_{1}, \dotsc, \ell_{d}} (-c(Q_{i}))^{\ell_{0}} \prod_{k = 1}^{d} Q_{i, k}^{\ell_{k}} \sum_{\relevant~j} \prod_{k = 1}^{d} K_{j, k}^{\ell_{k}} V_{j, t} \\
        &= \sum_{\ell} m_{\ell} \sum_{\ell_0+\ell_1 + \dotsc + \ell_{d} = \ell} \binom{\ell}{\ell_{0}, \ell_{1}, \dotsc, \ell_{d}} (-c(Q_{i}))^{\ell_{0}} \prod_{k = 1}^{d} Q_{i, k}^{\ell_{k}} \phi_{t}(i, \ell_{1}, \dotsc, \ell_{d})
    \end{align*}
    where we define the function $\phi_{t}(i, \ell_{1}, \dotsc, \ell_{d}) = \sum_{\relevant~j} \prod_{k = 1}^{d} K_{j, k}^{\ell_{k}} V_{j, t}$.
    Similarly, define the function
    \begin{equation*}
        \phi_{0}(i, \ell_{1}, \dotsc, \ell_{d}) = \sum_{\relevant~j} \prod_{k = 1}^{d} K_{j, k}^{\ell_{k}}
    \end{equation*}
    so that
    \begin{equation*}
        \sum_{j} P(Q_{i} \cdot K_{j} - c(Q_{i})) = \sum_{\ell} m_{\ell} \sum_{
        \ell_0+\ell_1 + \dotsc + \ell_{d} = \ell} \binom{\ell}{\ell_{0}, \ell_{1}, \dotsc, \ell_{d}} (-c(Q_{i}))^{\ell_{0}} \prod_{k = 1}^{d} Q_{i, k}^{\ell_{k}} \phi_{0}(i, \ell_{1}, \dotsc, \ell_{d}).
    \end{equation*}

    The following lemma describes how to build the appropriate data structures.
    
    \begin{restatable}{lemma}{ComputePhiJ}
        \label{lem:compute-phi-j}
        Let $\ell_{1}, \ldots, \ell_{d}$ be nonnegative integers.
        Let $0 \leq t \leq d$.
        Given matrices $Q, K, V$, there is a data structure with $O(n d + n \log n)$ preprocessing time that answers queries $\phi_{t}(i, \ell_{1}, \dotsc, \ell_{d})$ in $\bigtO{n^{1 - 1/d} \log(d\entryBound/\varepsilon)}$ time.
    \end{restatable}

    \begin{proof}
        We initialize two $\rsds$ data structures using \Cref{thm:range-searching-ds}, one with unweighted point set $\SET{K_{j}}$ and one with weighted point set $\SET{\left(K_{j}, \prod_{k = 1}^{d} K_{j, k}^{\ell_{k}} V_{j, t}\right)}_{j = 1}^{n}$.
        By \Cref{thm:range-searching-ds}, this requires $O(n \log n)$ preprocessing.
        Computing each weight requires $O(n d)$ time.
        
        Now, consider a query $\phi_{t}(i, \ell_{1}, \dotsc, \ell_{d})$ for some $i \in [n]$.
        We compute $s_{\max}^{(i)}$ using binary search with the first $\rsds$ data structure.
        Since $|Q_{i} \cdot K_{j}| \leq d \entryBound^2$ there are at most $O(d \entryBound^2/\log(1 + \varepsilon))$ values to search through.
        This requires $O(\log(d \entryBound/\varepsilon))$ queries which requires $\bigtO{n^{1 - 1/d} \log(d\entryBound/\varepsilon)}$ overall time by \Cref{thm:range-searching-ds}.
        The set of $j$ relevant to $Q_{i}$ is the set of $K_{j}$ such that $Q_{i} \cdot K_{j} \geq s_{\max}^{(i)} \log(1 + \varepsilon) - \log(n/\varepsilon)$.
        This can easily be captured by a simplex query with the half-space $Q_{i} \cdot x \geq s_{\max}^{(i)} \log(1 + \varepsilon) - \log(n/\varepsilon)$ and thus requires one query to the second $\rsds$ instance.
    \end{proof}

    Our data structure for \Cref{lem:small-dimension-data-structure} is simply the combination of all data structures that answer queries $\phi_{t}(i, \ell_{1}, \dotsc, \ell_{d})$.
    Since $P$ is degree $g$ and $\ell_{1} + \ell_{2} + \dotsc + \ell_{d} \leq \ell \leq g$, there are at most $(g + d)^{O(d)} = g^{O(d)}$ distinct tuples $\ell_{1}, \dotsc, \ell_{d}$ since $d$ is a constant.
    In particular, we can initialize all the necessary data structures to compute queries of $\phi_{t}$ in $\bigtO{g^{O(d)}(n d + n \log n)}$ time.
    
    We now show to compute an entry of $\hat{O}_{i, t}$.
    Note that numerator sums over $\ell$, tuples $\ell_{0}, \dotsc, \ell_{d}$ of which there are at most $g^{O(d)}$ summands.
    Each summand can be computed with one query to $\phi_{t}$ and $g^{O(d)}$ additional time.
    Since the denominator can be computed similarly (instead querying $\phi_{0}$) the total time to compute $\hat{O}_{i, t}$ is $\bigtO{g^{O(d)} n^{1 - 1/d} \log(d B/\varepsilon)}$.
\end{proof}

\begin{algorithm}[h!]
    \SetKwInOut{Input}{Input}\SetKwInOut{Output}{Output}\SetKwInOut{Parameters}{Parameters}
    \Input{vectors $Q, K, V \in [-\entryBound, \entryBound]^{n}$.}
    \Parameters{error parameter $\varepsilon$}
    \Output{$\hat{O}$ satisfying $\norm{\hat{O} - \att(q, k, v)}_{\infty} \leq 7 \varepsilon \entryBound$.}

    Compute $s_{\max}^{(i)}$ for all $i \in [n]$ using \Cref{thm:range-searching-ds}\;
    
    Compute $c(Q_{i}) \gets s_{\max}^{(i)} \log(1 + \varepsilon) - \log(n/\varepsilon)$ for all $i \in [n]$\;
    
    Compute a $g$-degree polynomial $P(x)$ for range $[0, 2 \log(n/\varepsilon)]$ using \Cref{lem:poly-mult-approx}\;
    
    Initialize the data structure for queries $\phi_{t}(i, \ell_{1}, \dotsc, \ell_{d})$ for all $0 \leq t \leq d$ using \Cref{lem:small-dimension-data-structure}\;
    
    Compute $\hat{O}_{i, t}$ for all $(i, t) \in [n] \times [d]$ using queries to \Cref{lem:small-dimension-data-structure}\;
    
    \Return $\hat{O}$\;
    \caption{$\appAttF(Q, K, V)$} 
    \label{alg:const-d-att-faster}
\end{algorithm}

We bound the running time of \Cref{alg:const-d-att-faster}.

\begin{restatable}{lemma}{AlgorithmRunningTime}
    \label{lem:const-d-time}
    $\appAttF$ (\Cref{alg:const-d-att-faster}) runs in time $\bigtO{n^{2 - 1/d} \cdot \polylog(\entryBound/\varepsilon)}$.
\end{restatable}

\begin{proof}
    We now analyze the running time.
    From \Cref{lem:poly-mult-approx}, we have
    \begin{equation*}
        g = \bigO{\max\left( \frac{\log(1/\varepsilon)}{\log(\log(1/\varepsilon)/\log(n/\varepsilon))}, \log(n/\varepsilon) \right)} = O(\log(n/\varepsilon)).
    \end{equation*}
    Then, to initialize all the necessary data structures, we invoke \Cref{lem:small-dimension-data-structure} a total of $d + 1$ times, thus requiring preprocessing time (recall $d$ is a constant)
    \begin{equation*}
        \bigtO{n \cdot \polylog(1/\varepsilon)}.
    \end{equation*}
    Then, computing all $\hat{O}_{i, t}$ requires time
    \begin{equation*}
        \bigtO{n g^{O(d)} \left( n^{1 - 1/d} \log(\entryBound/\varepsilon) \right)} = \bigtO{n^{2 - 1/d} \cdot \polylog(\entryBound/\varepsilon)}.
    \end{equation*}
\end{proof}

To conclude the proof of \Cref{m-thm:const-d-att-alg}, we run \Cref{alg:const-d-att-faster} with error parameter $\varepsilon' \leq \frac{\varepsilon}{7\entryBound}$.
We note that we can generalize our result to obtain an algorithm for computing Attention when the input matrices have low rank.

\ConstantRankAttentionAlg*

To prove \Cref{thm:low-rank-attention-alg}, we require the following standard result on computing a representation of low-rank matrices.

\begin{lemma}[e.g.,\ \cite{CMULinAlg, StanfordLinAlg}]
    \label{lemma:svd}
    Let $A$ be a $n \times d$ matrix of rank $r$ with entries in $[-\entryBound, \entryBound]$.
    Then, there is an $O(ndr)$ time algorithm computing an $n \times r$ matrix $U_{A}$ and a $d \times r$ matrix $V_{A}$ such that $A = U_{A} V_{A}^{\transpose}$.
    Furthermore, $U_{A}, V_{A}$ have entries bounded by $\poly(Bnd)$.
\end{lemma}

Suppose we are given $n \times d$ input matrices $Q, K$ of rank $r_{Q}, r_{K}$ respectively.
Then, we apply \Cref{lemma:svd} to compute $U_{Q}, V_{Q}, U_{K}, V_{K}$ in time $O(n d \max(r_{Q}, r_{K})) = O(n d)$.
Suppose without loss of generality $r_{Q} \leq r_{K}$.
Then, we compute
\begin{equation*}
    Q' = U_{Q}~,~K'^{\transpose} = V_{Q}^{\transpose} U_{K} V_{K}^{\transpose}
\end{equation*}
in time $O(r_{Q} r_{K} n) = O(n)$ and note that $Q', K'$ have entries bounded by $\poly(Bnd)$.

We then apply \Cref{m-thm:const-d-att-alg} to approximate $\att(Q', K', V) = \att(Q, K, V)$ which is an instance of $\attc(n, \min(r_{Q}, r_{K}), \poly(Bnd), \varepsilon)$ which requires time
\begin{equation*}
    \bigtO{n^{2 - 1/\min(r_{Q}, r_{K})} \cdot
    \polylog(\entryBound/\varepsilon)}
\end{equation*}
to compute an output $\hat{O}$ such that $\norm{\hat{O} - \att(Q, K, V)}_{\infty} \leq \varepsilon$.
This completes the proof of \Cref{thm:low-rank-attention-alg}.

\section{The Complexity of Attention Gradient Computation}
\label{sec:backwards-pass}

In this section, we leverage our algorithm for approximate attention computation to obtain the corresponding upper bounds for approximate attention gradient computation.
We begin by formalizing the notion of \textit{attention optimization}:

\begin{definition}[Attention Optimization]
\label{def:attention-optimization}
Given input matrices \( A_1, A_2, A_3, E \in \mathbb{R}^{n \times d} \) and \( Y \in \mathbb{R}^{d \times d} \), find a matrix \( X \in \mathbb{R}^{d \times d} \) that minimizes the objective:
\[
L(X) := \frac{1}{2} \left\| D(X)^{-1} A V - E \right\|_F^2,
\]
where \(A := \exp(A_1 X A_2^\top)\), \(V := A_3 Y \), and \( D(X) := \operatorname{diag}(A \mathbf{1}_n) \in \mathbb{R}^{n \times n} \). \footnote{\cite{alman2024fine} scale the Attention matrix $A$ by $d$ for training efficiency, becoming $A := \exp\left(\frac{A_1 X A_2^\top}{d}\right)$. 
Since our algorithms scale polylogarithmically with entry size, we can safely ignore this scaling term.}
\end{definition}

The gradient of the objective function $L(X)$ with respect to $X$ is then used to optimize the attention mechanism by iteratively adjusting $X$ to minimize $L(X)$.
Formally, we define the following approximate version of the gradient computation problem for attention optimization:

\begin{definition}[Approximate Gradient Computation for Attention Optimization $\aattlgc(n, d, \varepsilon)$]
\label{def:approx-gradient-comp}
Given \( A_1, A_2, A_3, E \in [-B, B]^{n \times d} \), \( Y \in [-B, B]^{d \times d} \), and $\varepsilon > 0$, 
compute a matrix \( g \in \R^{d \times d} \) such that
\[
    \left\| g - \odv{L(X)}{X} \right\|_\infty \leq \varepsilon.
\]

\end{definition}

\subsection{Notation}
Throughout this section we use the following notation. 
We overload the $\operatorname{diag}$ operator. In this section, the $\operatorname{diag}$ operator indicates turning all the non-diagonal entries to zero. 
The $\circ$ operator indicates entry-wise multiplication. The $\otimes$ operator denotes the Kronecker product, as defined by $Z[(i -1)n + \ell,(j-1)d + k] = X[i,j] \cdot Y[\ell, k]$ where $X, Y \in \mathbb{R}^{n \times d}$ and $Z \in \mathbb{R}^{n^2 \times d^2 }$.   The $\otimes_{r}$ operator denotes row-wise Kronecker product, as defined by $Z[i,(j-1)d + k] = X[i,j] \cdot Y[i, k]$ where $X, Y \in \mathbb{R}^{n \times d}$ and $Z \in \mathbb{R}^{n \times d^2 }$.
We use $e^{\langle i,j \rangle}$ as shorthand to denote $e^{ a_{1_i} \cdot a_{2_j}}$, where $a_{1_i}$ and $a_{2_j}$ are rows of $A_1$ and $A_2$ respectively.
If $M$ is a matrix, we use $M_{i}$ to denote the $i$-th row of $M$, $M_{*,i}$ to denote the $i$-th column of $M$. 
We use $M[i][j]$ to denote the $(i, j)$-th entry of $M$ (since our matrices have subscripts, the previous notation $M_{i, j}$ is confusing).

\subsection{Upper Bound on Attention Backward Computation}

We show that the backwards pass for approximate attention can be computed in time \(\tilde{O}\left(n^{2 - 1/d} \cdot \operatorname{polylog}(B / \varepsilon)\right) \) when $d = O(1)$.

\begin{theorem}[Formal \Cref{m-thm:backwards-pass-alg}]
    \label{thm:backwards-pass-alg}
    \(\aattlgc(n, d, B, \varepsilon)\) is reducible to $O(d)$ calls to \(\aattc(n, d, B, \frac{\varepsilon}{\Theta(ndB^{3})})\) using \(O(nd^2)\) time.
\end{theorem}

\begin{corollary}
    \label{cor:backwards-pass-alg}
    Let $d = O(1)$.
    There exists an algorithm that computes \( \aattlgc(n, d, B, \varepsilon) \) in time \( \tilde{O}\left(n^{2 - 1/d} \cdot \operatorname{polylog}(B / \varepsilon)\right) \).
\end{corollary}

\begin{proof}[Proof of \Cref{cor:backwards-pass-alg}]
    This follows directly from \Cref{thm:backwards-pass-alg} and \Cref{m-thm:const-d-att-alg}. 
\end{proof}

\begin{proof}[Proof of \Cref{thm:backwards-pass-alg}]
We begin by recalling the following definitions from \cite{alman2024fine}, which we will use to define the gradient computation formula. 

\begin{definition}
Let \( A_1, A_2 \in \R^{n \times d} \) be two matrices and let \( A = A_1 \otimes A_2 \in \R^{n^2 \times d^2} \).
Let $x \in \R^{d^2}$ be the vectorization of the matrix $X \in \R^{d \times d}$ in \Cref{def:attention-optimization}.
We define \( A_{j_0} \in \R^{n \times d^2} \) to be the \( n \times d^2 \) size sub-block of \( A \) consisting of rows \( \{ (j_0 - 1)n + j_1 \}_{j_1=1}^n \).
Let \( f(x) \) be the \( n \times n \) matrix whose \(j_0\)-th row, denoted \(f(x)_{j_0}\), is given by: 
\[
f(x)_{j_0} := (\langle \underbrace{\exp(A_{j_0} x)}_{n \times 1}, \underbrace{1_n}_{n \times 1} \rangle^{-1}
\underbrace{\exp(A_{j_0} x)}_{n \times 1})^\top.
\]
\end{definition} 

Note that \( f(x) = \exp(A_1XA_2^\top) \cdot \operatorname{diag}(\operatorname{exp}(A_1 XA_2^\top) \mathbf{1}_n) \). Therefore $f(x)Z$, where $Z$ is an $n \times d$ matrix, is evaluated by $\att(A_1, A_2, X)$.

\begin{definition}
Let \( Y \in \mathbb{R}^{d \times d} \) denote the matrix representation of \( y \in \mathbb{R}^{d^2} \) and \( Y_{*, i_0}  \) indicate the \(i_0\)-th column of \(Y\). \( h(y) \in \mathbb{R}^{n \times d} \) is defined as the matrix whose \( i_0 \)-th column is \( h(y)_{i_0} \), which is defined as follows:
\[
h(y)_{i_0} := \underbrace{A_3}_{n \times d} \underbrace{Y_{*, i_0}}_{d \times 1}.
\]
\end{definition}

Note that throughout this section, we occasionally use $h$ as a shorthand for $h(y)$. It is clear that $h(y)$ can be computed in $\TMUL(n,d,d)$ time. 

\begin{definition} Let \(c(x)\) be an \(n \times d\) matrix defined as follows:
\[
 \underbrace{c(x)}_{n \times d} = \underbrace{f(x)}_{n \times n} \underbrace{h(y)}_{n \times d} - \underbrace{E}_{n \times d}.
\]
\end{definition}

We can approximate $c(y)$ by evaluating $\att(A_1X, A_2, h(y))$ to get $f(x)h(y)$, then subtracting $E$ which takes $O(nd)$ time.

From \cite{alman2024fine} we have the following formula for attention gradient computation:
\begin{align*}
\odv{L(x)}{x}
&= A_1^\top[f(x) \circ (c(x) h(y)^\top)] A_2 - A_1^\top f(x) \operatorname{diag}[ f(x) c(x) h(y)^\top ] A_2 \\
&= A_1^\top[f(x) \circ ((f(x)h(y) - E) h(y)^\top)] A_2 - A_1^\top f(x) \operatorname{diag}[ f(x) c(x) h(y)^\top ] A_2 \\
&= A_1^\top[f(x) \circ (f(x)h(y)h(y)^\top)] A_2 - A_1^\top[f(x) \circ (E h(y)^\top)] A_2 \\ 
&\mathrel{\phantom{=}} - A_1^\top f(x) \operatorname{diag}[ f(x) c(x) h(y)^\top ] A_2.
\end{align*}
The first line comes from the characterization of the gradient as $\odv{L(x)}{x} = A_{1}^{\transpose} p(x) A_{2}$ where $p(x) = p_1(x) - p_2(x)$ (see Appendix D.4-D.6 of \cite{alman2024fine}).
In the notation of \cite{alman2024fine}, the first term corresponds to $p_1(x) := f(x) \circ q(x) := f(x) \circ (c(x) h(y)^{\transpose})$.
The second term corresponds to $p_2(x)$ which is an $n \times n$ matrix whose $j_0$-th column is $f(x)_{j_0} f(x)_{j_0}^{\transpose} q(x)_{j_0} := f(x)_{j_0} f(x)_{j_0}^{\transpose} c(x) h(y)_{j_0}^{\transpose}$. Note that $p_2(x) := f(x) \operatorname{diag}[ f(x) q(x) ] = f(x) \operatorname{diag}[ f(x) c(x) h(y)^\top ]$.
Note that $q(x) = c(x) h(y)^\top$ is notation in \cite{alman2024fine} which we do not use here.

Let us denote 
\begin{align*}
    B_1 &:= [f(x) \circ (f(x)h(y)h(y)^\top)] A_2,\\
    B_2 &:= [f(x) \circ (E) h(y)^\top)] A_2,\\
    B_3 &:= f(x) \operatorname{diag}[ f(x) c(x) h(y)^\top] A_2.
\end{align*}

We now have the following formula which can clearly be computed in $O(nd)$ time if given $B_1, B_2$, and $B_3$:
\[
\odv{L(x)}{x}
= \underbrace{A_1^\top}_{d \times n}\underbrace{B_1}_{n \times d} - \underbrace{A_1^\top}_{d \times n}\underbrace{B_2}_{n \times d} - \underbrace{A_1^\top}_{d \times n}\underbrace{B_3}_{n \times d}.
\]
Note that for each attention computation we perform in order to evaluate the attention gradient, we do with $\varepsilon_2 = \frac{\varepsilon}{\poly(d,B)n}$ additive error.

\paragraph{Computing $B_3$.} 
Given $f(x), c(x)$, and $h(y)$, we can approximate $B_3$ using a series of matrix multiplications and attention computations, which are illustrated below in the following equations. \( C_i \) denotes the intermediate matrix products from each of these matrix multiplications/attention computations.
We compute an approximation of $B_3$ as follows: 
\begin{align*}
B_3
&= f(x) \operatorname{diag}[ \underbrace{f(x) }_{n \times n} \underbrace{c(x) }_{n \times d} h(y)^\top ] A_2 \\
&= f(x)\operatorname{diag}[\underbrace{C_1}_{n \times d} \underbrace{h(y)^\top }_{d \times n}] A_2  \\
&= f(x) \underbrace{C_2 }_{n \times n} \underbrace{A_2 }_{n \times d} \\
&= f(x) \underbrace{C_3 }_{n \times d}.
\end{align*}
We begin by computing \( C_1 = f(x)c(x) \) by evaluating $\att(A_1X, A_2, c(x))$.
Next, we compute \( C_2 = \operatorname{diag}[C_1 h(y)^\top] \), which consists of the diagonal of the matrix product \( C_1 h(y)^\top \). Since we only need the diagonal entries, this step takes \( O(nd^2) \) time.
We then compute \( C_3 = C_2 A_2 \). As \( C_2 \) is a diagonal matrix, this matrix multiplication can be performed in \( O(nd) \) time.
Finally, we compute \( B_3 = f(x) C_3 \) by evaluating $\att(A_1X, A_2, C_3)$. 

We argue that our computed output is a good approximation of $B_3$.
Let $\widetilde{B_3}$ denote the computed matrix. For any matrix $Z$, $\widetilde{Z}$ indicates an approximation of $Z$ derived by a step in our algorithm. 
Then,
\begin{align*}
\norm{B_3 - \widetilde{B_3}}_{\infty} 
&\leq \left\| f(x)C_3 - \attc(A_1X, A_2, \widetilde{C_3}) \right\|_\infty \\
&\leq \left\| f(x)C_3 - f(x)\widetilde{C_3} \right\|_\infty + \varepsilon_2 \\
&\leq \left\| C_3 - \widetilde{C_3}  \right\|_\infty + \varepsilon_2 \\
&= \left\| \text{diag}[C_1h(y)^\top] A_2 - \text{diag}[\widetilde{C_1} h(y)^\top]A_2  \right\|_\infty + \varepsilon_2  \\
&= \left\| \left[ \text{diag}[C_1h(y)^\top] 
 - \text{diag}[ \widetilde{C_1} h(y)^\top]\right] A_2  \right\|_\infty + \varepsilon_2  \\
&\leq  \left\| A_2 \right\|_\infty \left\| \text{diag}[C_1h(y)^\top] 
 - \text{diag}[ \widetilde{C_1} h(y)^\top] \right\|_\infty + \varepsilon_2  \\
 &\leq  \left\| A_2 \right\|_\infty \left\| C_1h(y)^\top - \widetilde{C_1} h(y)^\top \right\|_\infty + \varepsilon_2  \\
&\leq d \left\| A_2 \right\|_\infty \left\| h(y) \right\|_\infty \left\| C_1 - \widetilde{C_1} \right\|_\infty + \varepsilon_2 \\
&\leq d \left\| A_2 \right\|_\infty \left\| h(y) \right\|_\infty \left\| f(x)c(x) - \attc(A_1X, A_2, \widetilde{c(x)}) \right\|_\infty + \varepsilon_2 \\
&\leq d \left\| A_2 \right\|_\infty \left\| h(y) \right\|_\infty ( \varepsilon_2  + \left\| f(x)c(x) - f(x)\widetilde{c(x)} \right\|_\infty) + \varepsilon_2 \\
&\leq d \left\| A_2 \right\|_\infty \left\| h(y) \right\|_\infty ( \varepsilon_2  + \left\| c(x) - \widetilde{c(x)} \right\|_\infty) + \varepsilon_2 \\
&\leq 2dB^2 \varepsilon_2 + \varepsilon_2.
\end{align*}
Above, step 1 follows from how our algorithm approximates $B_3$, step 2 follows from our $\varepsilon_2$-error approximation of attention and the triangle inequality, step 3 follows from the fact that $f(x)$ is a stochastic matrix and distributivity of matrix multiplication, step 4 follows from our definition of $C_3$, step 5 follows from the distributivity of matrix multiplication, and step 6 follows from basic properties of the $\infty$-norm and diagonal matrices.
Step 7 follows from the fact that the diag operator simply zeroes out the off-diagonal entries, making the off-diagonal elements of $C_1h(y)^\top$ and $\widetilde{C_1}h(y)^\top$ identical. 
Step 8 follows from basic properties of the $\infty$-norm, step 9 follows from how our algorithm approximates $C_1$, step 10 follows from the triangle inequality and our $\varepsilon_2$ approximation of attention, step 11 follows from similar arguments as steps 9 and 10, and step 12 follows from entry bounds.

\paragraph{Computing $B_1$.}
We now show how to compute $B_1$. We begin by noting that $B_1 = \sum_{p=0}^{d}(f(x) (h(y)_{*,p}  \otimes_{r} A_2)) \otimes_{r} (f(x)h(y))_{*,p}$, a fact we will prove later. Using this fact, we can compute \(B_1\) efficiently, as illustrated in the following:
\begin{align*} 
B_1 &= \sum_{p=0}^{d}(f(x) (h(y)_{*,p}  \otimes_{r} A_2)) \otimes_{r} (f(x)h(y))_{*,p} \\
&= \sum_{p=0}^{d}(f(x) (\underbrace{h(y)_{*,p}}_{n \times 1}  \otimes_{r} \underbrace{A_2}_{n \times d}) \otimes_{r} C_{5_{*,p}} \\
&= \sum_{p=0}^{d}(f(x) \underbrace{C_{6,p}}_{n \times d}) \otimes_{r} C_{5_{*,p}} \\
&= \sum_{p=0}^{d}\underbrace{C_{7,p}}_{n \times d}\otimes_{r} \underbrace{C_{5_{*,p}}}_{n \times 1} \\
&= \sum_{p=0}^{d}C_{8,p}.
\end{align*}

We begin by approximating \( C_5 = f(x) h(y) \) by evaluating $\att(A_1X, A_2, h(y))$.
Next, for each \( 1 \leq p \leq d \), we compute the matrix \( C_{6,p} = h(y)_{*,p} \otimes_{r} A_2 \). Each matrix requires \( O(nd) \) time to compute, so constructing all \( d \) matrices incurs a total cost of \( O(nd^2) \).

We then compute each matrix \( C_{7,p} = f(x) C_{6,p} \) by evaluating $\att(A_1X, A_2, C_{6,p})$ across all \( p \in [d]\).
Computing the row-wise Kronecker products \( C_{8,p} = C_{7,p} \otimes_{r} C_{5_{*,p}} \) takes \( O(nd) \) time for each \( p \in [d] \), totaling \( O(nd^2) \). Finally, summing over all \( C_{8,p} \) requires an additional \( O(nd^2) \) time.

We argue that our algorithm returns a close approximation of $B_1$. Let $\widetilde{B_1}$ indicate our computation of $B_1$. For any matrix $Z$, $\widetilde{Z}$ indicates an approximation of $Z$ derived by a step in our algorithm. 
\begingroup
\allowdisplaybreaks
\begin{align*}
\left\| \widetilde{B_1} - B_1 \right\|_\infty 
&= \left\| \sum_{p=0}^{d} \widetilde{C_{8,p}} - \sum_{p=0}^{d} C_{8,p} \right\|_\infty \\
&\leq d \max_p \left\{ \left\| \widetilde{C_{8,p}} - C_{8,p} \right\|_\infty \right\} \\
&\leq d \max_p \left\{ \left\| \widetilde{C_{7,p}} \otimes_r \widetilde{C_{5_{*,p}}}  - C_{7,p} \otimes_r C_{5_{*,p}} \right\|_\infty \right\} \\
&\leq d \max_p \left\{ \left\| \widetilde{C_{7,p}} - C_{7,p} \right\|_\infty \left\| \widetilde{C_{5_{*,p}}} - C_{5_{*,p}} \right\|_\infty \right. \\
&\qquad + \left\| \widetilde{C_{7,p}} - C_{7,p} \right\|_\infty \left\| C_{5_{*,p}} \right\|_\infty \\
&\qquad \left. + \left\| \widetilde{C_{5_{*,p}}} - C_{5_{*,p}} \right\|_\infty \left\| C_{7,p} \right\|_\infty \right\} \\
&= d \max_p \left\{ \left\| \attc(A_1X, A_2, C_{6,p}) - f(x)C_{6,p} \right\|_\infty \left\| \attc(A_1X, A_2, h(y))_{*,p} - C_{5_{*,p}} \right\|_\infty \right. \\
&\qquad + \left\| \attc(A_1X, A_2, C_{6,p}) - f(x)C_{6,p} \right\|_\infty \left\| C_{5_{*,p}} \right\|_\infty \\
&\qquad \left. + \left\| \attc(A_1X, A_2, h(y))_{*,p} - C_{5_{*,p}} \right\|_\infty \left\| C_{7,p} \right\|_\infty \right\} \\
&\leq d \max_p \left\{ \varepsilon_2^2 + \varepsilon_2 \left\| (f(x)h(y))_{*,p} \right\|_\infty + \varepsilon_2 \left\| f(x) (h(y)_{*,p} \otimes_{r} A_2) \right\|_\infty \right\} \\
&\leq d \max_p \left\{ \varepsilon_2^2 + \varepsilon_2 \left\| h(y) \right\|_\infty 
+ \varepsilon_2 \left\| h(y)_{*,p} \otimes_{r} A_2 \right\|_\infty \right\} \\
&\leq d \max_p \left\{ \varepsilon_2^2 + \varepsilon_2  \left\| h(y) \right\|_\infty 
+ \varepsilon_2 \left\| h(y)_{*,p} \right\|_\infty \left\| A_2 \right\|_\infty \right\} \\
&\leq d \left( \varepsilon_2^2 + \varepsilon_2 B^2 + \varepsilon_2 B^3 \right) = d\varepsilon_2^2 + d\varepsilon_2 B + d\varepsilon_2 B^2.
\end{align*}%
\endgroup
Step 1 follows from our definition of $C_{8,p}$, step 2 follows from the triangle inequality, and step 3 follows from how we define $C_{8,p}$. 
Step 4 follows from analyzing the entry-wise error in the row-wise Kronecker product. 
Let \( a = C_{7,p}[i][j] \), \( b = C_{5_{*,p}}[i][j] \), and let \( e_1 \) and \( e_2 \) denote the entry-wise approximation errors in \( C_{7,p}[i][j] \) and \( C_{5_{*,p}}[i][j] \), respectively. 
Then the approximated entry is \( \widetilde{c} = (\widetilde{C_{7,p}} \otimes_r \widetilde{C_{5_{*,p}}})[i][j] = (a + e_1)(b + e_2) = ab + b e_1 + a e_2 + e_1 e_2 \). 
Therefore, the entry-wise error in the approximation is \( \widetilde{c} - c = b e_1 + a e_2 + e_1 e_2 \), where \( (c = C_{7,p} \otimes_r C_{5_{*,p}})[i][j] \). 

Step 5 follows from how our algorithm approximates $C_{7,p}$ and $C_{5_{*,p}}$. Step 6 follows from the fact that $\widetilde{C_6} = C_6$ and our $\epsilon_2$ approximation of the attention computation.
Step 7 follows from the fact that $f(x)$ is a stochastic matrix, step 8 is based on the linearity of the Kronecker product, and step 9 follows from entry bounds.

We defined $B_1 := [f(x) \circ (f(x)h(y)h(y)^\top)] A_2$. We now show %
We begin by noting that the format of each entry of $B_{1}$ is as follows, where $1 \leq i \leq n$ and $1 \leq j \leq d$:
\begin{align*}
B_{1}[i,j] 
&= \sum_{\ell=0}^{n} \frac{e^{\langle i, \ell \rangle}}{\sum_{k=0}^{n}e^{\langle i, k \rangle}}
\left[
\sum_{m=0}^{n} \frac{e^{\langle i, m \rangle}}{\sum_{k=0}^{n} e^{\langle i, k \rangle}} \sum_{p=0}^{d} h[\ell, p] h[m, p]
\right]
A_2[\ell, j] \\
&= \sum_{p=0}^{d}
\sum_{\ell=0}^{n} \frac{e^{\langle i, \ell \rangle}}{\sum_{k=0}^{n}e^{\langle i, k \rangle}}
\left[
\sum_{m=0}^{n} \frac{e^{\langle i, m \rangle}}{\sum_{k=0}^{n} e^{\langle i, k \rangle}}  h[m, p]
\right]
h[\ell, p] A_2[\ell, j].
\end{align*}

We now compute the sum \( \sum_{p=0}^{d} C_{8,p} \) and verify that
\[
\left[ \sum_{p=0}^{d} C_{8,p} \right][i,j] = B_2[i,j].
\]

Let \( C_5 = f(x) h(y) \). For \( 1 \leq i \leq n \) and \( 1 \leq p \leq d \), we have:
\[
C_5[i, p] = \sum_{m=0}^{n} \frac{e^{\langle i, m \rangle}}{\sum_{k=0}^{n} e^{\langle i, k \rangle}} h(y)[m, p].
\]

Let \( C_{6,p} = h(y)_{*,p} \otimes_{r} A_2 \). For \( 1 \leq \ell \leq n \) and \( 1 \leq j \leq d \), this gives:
\[
C_{6,p}[\ell, j] = h(y)[j, \ell] \cdot A_2[\ell, j].
\]

We define \( C_{7,p} = f(x) C_{6,p} \), so:
\[
C_{7,p}[i, j] = \sum_{\ell=0}^{n} \frac{e^{\langle i, \ell \rangle}}{\sum_{k=0}^{n} e^{\langle i, k \rangle}} h(y)[j, \ell] A_2[\ell, j].
\]

Let \( C_{8,p} = f(x) C_{7,p} \otimes_{r} C_{5_{*,p}} \). Then for \( 1 \leq i \leq n \), \( 1 \leq j \leq d \):
\begin{align*}
C_{8,p}[i, j] 
&= \left( \sum_{\ell=0}^{n} \frac{e^{\langle i, \ell \rangle}}{\sum_{k=0}^{n} e^{\langle i, k \rangle}} h(y)[j, \ell] A_2[\ell, j] \right)
    \left( \sum_{m=0}^{n} \frac{e^{\langle i, m \rangle}}{\sum_{k=0}^{n} e^{\langle i, k \rangle}} h(y)[m, p] \right) \\
&= \sum_{\ell=0}^{n} \sum_{m=0}^{n} 
    \frac{e^{\langle i, \ell \rangle}}{\sum_{k=0}^{n} e^{\langle i, k \rangle}}
    \cdot \frac{e^{\langle i, m \rangle}}{\sum_{k=0}^{n} e^{\langle i, k \rangle}}
    \cdot h(y)[m, p] \cdot h(y)[\ell, p] \cdot A_2[\ell, j].
\end{align*}

Summing over all \( p \), we recover:
\[
B_2[i, j] = \sum_{p=0}^{d} C_{8,p}[i, j].
\]

\paragraph{Computing $B_2$.} 
We begin by noting that $B_2 = \sum_{p=0}^{d} [f(x)(h(y)_{*,p} \otimes_{r} A_2)]\otimes_{r} E_{*,p}$, a fact that we will prove later on. Using this fact, we use the following procedure to compute an approximation of $B_2$:

\begin{align*}
B_2 
&= \sum_{p=0}^{d} [f(x)\underbrace{(h(y)_{*,p}}_{n \times 1} \otimes_{r} \underbrace{A_2}_{n \times d})]\otimes_{r} E_{*,p} \\
&= \sum_{p=0}^{d} [f(x)\underbrace{C_{9,p}}_{n \times d}]\otimes_{r} E_{*,p} \\
&= \sum_{p=0}^{d} \underbrace{C_{10,p}}_{n \times d}\otimes_{r} \underbrace{E_{*,p}}_{n \times 1} \\
&= \sum_{p=0}^{d} \underbrace{C_{11,p}}_{n \times d}.
\end{align*}

We start by approximating the set of $d$ matrices, $C_{9,p} = h(y)_{*,p} \otimes_{r} A_2$. For each $1 \leq p \leq d$, computing $C_{9,p}$ takes $O(nd)$ time, so this takes $O(nd^2)$ time in total. We approximate each $C_{10,p} = f(x)C_{9,p}$ by evaluating $\att(A_1X, A_2, C_{9,p})$. Next, we compute all $C_{11,p} = C_{10,p} \otimes_{r} E_{*,p}$ which takes $O(nd^2)$ time in total. Finally, summing over $C_{11,p}$ takes $O(nd^2)$ time.

We now analyze the error from approximating $B_2$ using the method we just described. For any matrix $Z$, $\widetilde{Z}$ indicates an approximation of $Z$ derived by a step in our algorithm. 
\begin{align*}
\left\| B_2 - \widetilde{B_2} \right\|_\infty 
&= \left\| \sum_{p=0}^{d} C_{11,p} - \sum_{p=0}^{d}\widetilde{C_{11,p}}  \right\|_\infty \\
&\leq d \max_p \left\{ \left\| C_{11,p} - \widetilde{C_{11,p}}  \right\|_\infty \right\} \\
&= d \max_p \left\{ \left\| C_{10,p} \otimes_{r} E_{*,p} - \widetilde{C_{10,p}}  \otimes_{r} E_{*,p} \right\|_\infty \right\} \\
&= d \max_p \left\{ \left\| [C_{10,p} -\widetilde{C_{10,p}}]\otimes_{r} E_{*,p} \right\|_\infty \right\} \\
&\leq d \max_p \left\{ \left\| E_{*,p} \right\|_\infty \left\|  C_{10,p}  - \widetilde{C_{10,p}} 
\right\|_\infty \right\} \\
&= d \max_p \left\{ \left\| E_{*,p} \right\|_\infty \left\| f(x)(h(y)_{*,p} \otimes_{r} A_2) - \attc(A_1X, A_2, h(y)_{*,p} \otimes_{r} A_2) 
\right\|_\infty \right\} \\
&\leq d \max_p \left\{ \varepsilon_2 \left\| E_{*,p} \right\|_\infty \right\} \\
&\leq d \varepsilon_2 B.
\end{align*}

Above, step 1 follows from our definition of $C_{11,p}$, step 2 is follows from the triangle inequality, and step 3 follows from our definition of $C_{11,p}$. 
Step 4 follows from the linearity of the row-wise Kronecker product and step 5 follows from the fact that the row-wise Kronecker product scales every element in $C_{10,p}$ by an element in $E_{*,p}$. 
Step 6 follows from how we approximate $C_{10,p}$ in our algorithm, step 7 follows from our $\varepsilon_2$-error approximation of attention, and step 8 follows from our defined entry bounds.

We defined $B_2 :=[f(x) \circ (E) h(y)^\top)] A_2$.
Finally, we show that $B_2 = \sum_{p=0}^{d} [f(x)(h(y)_{*,p} \otimes_{r} A_2)]\otimes_{r} E_{*,p}$, which can be proven by showing that $B_2[i,j] = \sum_{p=0}^{d} C_{11,p}[i,j]$ for all $1 \leq i \leq n$ and $1 \leq j \leq d$.
We note the following:
\begin{align*}
B_{2}[i,j] 
&= \sum_{\ell=0}^{n} \frac{e^{\langle i, \ell \rangle}}{\sum_{k=0}^{n}e^{\langle i, k \rangle}} \left[ \sum_{p=0}^{d} E[i,p]h[\ell, p] \right] A_2[\ell, j] \\
&= \sum_{p=0}^{d} \sum_{\ell=0}^{n} \frac{e^{\langle i, \ell \rangle}}{\sum_{k=0}^{n}e^{\langle i, k \rangle}}  E[i,p]h[\ell, p] A_2[\ell, j],
\end{align*}
and it is clear that the following is true:
\[
C_{11,p}[i,j] = \sum_{\ell=0}^{n} \frac{e^{\langle i, \ell \rangle}}{\sum_{k=0}^{n}e^{\langle i, k \rangle}}  E[i,p]h[\ell, p] A_2[\ell, j].
\]

\paragraph{Bounding Approximation Error.} 

Now all that is left is to show our procedure gives us an approximation of the gradient with $\varepsilon$ additive error. 
Recall that we did all the attention calculations with $\varepsilon_2= \frac{\varepsilon}{\poly(d,B)n}$ additive error.  
Let $\widetilde{\odv{L(x)}{x}}$ denote the matrix our procedure returns and let $\widetilde{c(x)}$ be the approximation of $c(x)$ given by $\att(A_1X, A_2, h(y))$.
\begin{align*}
\left\| \odv{L(x)}{x} - \widetilde{\odv{L(x)}{x}} \right\|_{\infty} 
&= \left\|  A_1^\top B_1 - A_1^\top B_2 - A_1^\top B_3 - (A_1^\top \sum_{p=0}^{d}C_{8,p} - A_1^\top \sum_{p=0}^{d}C_{11,p} - A_1^\top f(x)C_3) \right\|_{\infty} \\
&= \left\|  A_1^\top \left[(B_1 - \sum_{p=0}^{d}C_{8,p}) + (B_2 - \sum_{p=0}^{d}C_{11,p}) + (B_3 - f(x)C_3)\right] \right\|_{\infty}  \\
&\leq n\left\|  A_1^\top \right\|_{\infty} \left\| (B_1 - \sum_{p=0}^{d}C_{8,p}) + (B_2 - \sum_{p=0}^{d}C_{11,p}) + (B_3 - f(x)C_3) \right\|_{\infty} \\
&\leq n\left\|  A_1^\top \right\|_{\infty} \left[\left\| B_1 - \sum_{p=0}^{d}C_{8,p} \right\|_{\infty} + \left\| B_2 - \sum_{p=0}^{d}C_{11,p}\right\|_{\infty} + \left\| B_3 - f(x)C_3 \right\|_{\infty} \right] \\
&\leq nB(\left( d\varepsilon_2^2 + d\varepsilon_2 B + d\varepsilon_2 B^2 \right) + d \varepsilon_2 B + (2dB^2 \varepsilon_2 + \varepsilon_2)) \\
&= O(n d B^{3} \varepsilon_2)= \varepsilon.
\end{align*}
Above, steps 1 and 2 follow from definitions and rearranging terms, step 3 follows from basic properties of the $\infty$-norm, step 4 follows from the triangle inequality, and step 5 was justified previously.

\end{proof}

\section{New Lower Bounds for Attention}
\label{sec:hardness}

In this section, we prove \Cref{m-thm:super-const-d-lb} which shows Attention is hard even with $d = 2^{\Theta(\log^* n)}$ and \Cref{m-thm:poly-d-lb} which shows that the standard algorithm is optimal for $d = \poly(n)$.
We begin with a generic self-reduction (\Cref{lem:de-normalization}) that shows it suffices to prove lower bounds for Attention without normalization.
We also prove \Cref{thm:ov-lb-constant} which shows that Attention is hard for $d = \Omega(\log n)$ even for constant entry size.

Recall that in the attention computation $\att(Q,K,V) = D^{-1} AV$, the diagonal matrix $D^{-1}$ applies a normalization to each row of $A$.
In our reductions, however, it is necessary to work directly with the unnormalized entries of $A$.
As a key lemma, we show that given oracle access to $\attc$ with $\varepsilon$-additive error approximation, one can approximately recover the row sums of $A$ up to $O(\varepsilon)$-\emph{multiplicative} errors, hence recovering the unnormalized entries of $A$.
Specifically, if $S_i$ is the actual row sum of the $i$-th row of $A$, then the reduction computes an approximation $\hat{S}_i$ such that
\[
|\hat S_i - S_i| < O(\varepsilon)S_i.
\]
It turns out that multiplicative error approximation on the row sums is sufficient for our lower bound proofs.

\begin{restatable}{lemma}{denormalization}
\label{lem:de-normalization}
    Let $0 < \varepsilon = O(1)$.
    Given matrices $Q, K \in [-B, B]^{n \times d}$ with $B \geq 1$, we can estimate the row sums of $A = \exp(QK^\top)$ up to $O(\varepsilon)$-multiplicative error in time
    \[
        O((\log \log n + \log (d B / \varepsilon))\TATTC(n+1, d+1, B, \varepsilon)).
    \]
\end{restatable}

\begin{proof}
    We use a parallel binary search approach to estimate the row sums.
    In order to implement parallel binary search, it suffices to perform the following task $\cT$:
    
    Given an array of numbers $\mathbf c = [c_1, \dotsc, c_n]^\top$, output an array $\mathbf b \in \{0, 1\}^n$ such that if $S_i \ge (1+\varepsilon)c_i$, then $b_i = 1$; if $S_i \le (1-\varepsilon)c_i$, then $b_i = 0$.
    Otherwise, $b_i$ can be arbitrary. 

    Indeed, at each round we let $c_i := (1+\varepsilon)^{f_i - 1}$ for some $f_i$.
    We use the indicator $b_i = 1$ to perform binary search for the smallest $f_i$ such that $(1+\varepsilon)^{f_i} \geq S_i$ for all $i$.
    Such an $f_i$ gives the guarantee that $S_i \leq (1+\varepsilon)^{f_i} < (1+\varepsilon)S_i$, which is an $\varepsilon$-multiplicative approximation of $S_i$.
    Note that the value of each row sum $S_i$ belongs to the range $[n \exp(-B^2 d), n \exp(B^2 d)]$, so we just need to binary search for the correct $f_i \in [\log_{1+\varepsilon}(n \exp(-B^2 d)), \log_{1+\varepsilon}(n \exp(B^2 d))]$.
    Therefore, the number of rounds for binary search (i.e., for performing the task $\cT$) is given by
    \[O(\log_2 \log_{1+ \varepsilon} (n\exp(2B^2 d)) = O(\log \log n + \log (d B / \varepsilon)).
    \]

    It now remains to show how to perform the task $\cT$.
    We claim the following:
    \begin{claim}\label[claim]{claim:denorm-lemma-subtask}
        The task $\cT$ can be completed with one oracle call to $\attc(n+1, d+1, B, \varepsilon/100)$ and $O(nd)$ additional time.
    \end{claim}
    \begin{proof}
         We create the following matrices as inputs to the oracle $\attc(n+1, d+1, B, \varepsilon)$:
        \[
            Q' := \begin{bmatrix}
                \ln \mathbf{c} & Q\\
                0 & \mathbf{0}_d^\top
            \end{bmatrix}, 
            K' := \begin{bmatrix}
                1 & \mathbf{0}_d^\top\\
                \mathbf{0}_n & K
            \end{bmatrix}, \,
            V' := \begin{bmatrix}
                0 & 0 & \cdots  &   0\\
                 1 & 0 & \cdots  &   0\\
                \vdots & \vdots & \ddots & \vdots  \\
                1 & 0 & \cdots  & 0
            \end{bmatrix}.
        \]
        Then,
        \[
         Q'K'^\top = \begin{bmatrix}
             \ln \mathbf c & QK^\top\\
             0 & \mathbf 0_d^\top
         \end{bmatrix},
        \]
        so the $(i, 1)$-th entry of $\att(Q', K', V') = D'^{-1} A' V'$ would be 
        \[
            o_i = \frac{S_i}{c_i + S_i}. 
        \]
        Assume we have an $(\varepsilon / 100)$-additive approximation of $o_i$ (denoted by $\hat{o}_i$).
        Then, we set $b_i = 1$ if $\hat{o}_i \ge \frac{1}{2}$ and $b_i = 0$ otherwise.
        We now show that all entries of $\mathbf b$ are correctly set. If $S_i \ge (1+\varepsilon) c_i$, then
        \begin{align*}
                \hat{o}_i \geq o_i - \varepsilon / 100 \geq \frac{S_i}{c_i + S_i} - \varepsilon / 100 \geq \frac{1+\varepsilon}{2+\varepsilon} - \varepsilon / 100 > \frac{1}{2} \text{.}
        \end{align*}
        On the other hand, if $S_i \le (1-\varepsilon) c_i$, then 
        \begin{align*}
                \hat{o}_i \leq o_i + \varepsilon / 100 \leq \frac{S_i}{c_i + S_i} + \varepsilon / 100 \leq \frac{1-\varepsilon}{2-\varepsilon}+\varepsilon / 100 < \frac{1}{2} \text{.}
        \end{align*}
        In the first inequality, we use $\frac{1 + \varepsilon}{2 + \varepsilon} > \frac{1}{2} + \frac{\varepsilon}{6}$ and in the second we use $\frac{1 - \varepsilon}{2 - \varepsilon} < \frac{1}{2} - \frac{\varepsilon}{6}$.
        Thus, the algorithm will output $b_i = 1$ in the former case and $b_i = 0$ in the latter case, as desired.
    \end{proof}
    This completes the proof of \Cref{lem:de-normalization}.
\end{proof}

\subsection{Lower Bound for Attention with Small Head Dimension}
In this section, we show via a reduction from the $\mip$ problem that $\attc(n, d, B, \varepsilon)$ requires $n^{2 - o(1)}$ time when $d = 2^{\Omega(\log^* n)}$, $B = \poly(n)$, and $\varepsilon = O(1)$ additive approximation error.
In particular, we note that we are able to compute $\mip$ exactly even with oracle access to $\attc$ that allows $\varepsilon=O(1)$ additive error.

\begin{restatable}{lemma}{ZMIPAttReduction}
    \label{lem:z-mip-att-reduction}
    Let $\varepsilon > 0$.
    $\mip(n, d, B)$ can be computed exactly in time
    \[
    O((\log \log n + \log (dB / \varepsilon))\TATTC(n+1, d+1, O(B \log n), \varepsilon)).
    \]
\end{restatable}

\begin{proof}
    Given a $\delta$, we choose a $C = C(\delta)$ and set $d = 2^{C\log*(n)}$. Let $\cA = \{a_1, \dotsc, a_n\}, \cB = \{b_1, \dotsc, b_n\} \subseteq \Z^d$ be two sets of $d$-dimensional integer-valued vectors with entries bounded by $B \geq 1$.
     Let $k = \ln n$ and we choose the smallest integer $C > 0$ such that 
    \[
        0.5 C > 1 + \log_n(1+\varepsilon) \quad \text{and} \quad -0.5C < \log_n (1 - \varepsilon).
    \]
   
    Define the following matrices $Q, K \in \R^{n \times d}$: 
    \begin{equation}\label{eqn:max-ip-reduction-qk-defn}
        Q := \begin{bmatrix}
            \horzbar & a^{\top}_{1} & \horzbar \\
            \horzbar & a^{\top}_{2} & \horzbar \\
                     &  \vdots    &          \\
            \horzbar & a^{\top}_{n} & \horzbar
        \end{bmatrix}, \,
        K := kC \cdot \begin{bmatrix}
            \horzbar & b^{\top}_{1} & \horzbar \\
            \horzbar & b^{\top}_{2} & \horzbar \\
                     &  \vdots    &          \\
            \horzbar & b^{\top}_{n} & \horzbar
        \end{bmatrix}.
    \end{equation}
    By \Cref{lem:de-normalization}, we get the $(1 \pm \varepsilon)$-multiplicative approximations of the row sums of $\exp(QK^\top)$ in time
    \[
        O((\log \log n + \log (k^2 C^2 B^2 d / \varepsilon)) \TATTC(n+1, d+1, kcB, \varepsilon)).
    \]
    Here, note that $kCB = O(B \log n)$.
    Note that the $i$-th row sum is given by
    \[
        S_i = \sum_{j = 1}^n e^{kC(a_i \cdot b_j)} = \sum_{j = 1}^n n^{C(a_i \cdot b_j)}.
    \]
    Let $S_i'$ be the $(1 \pm \varepsilon)$-multiplicative approximation for $S_i$ and 
    let $M_i := \max_j a_i \cdot b_j$ (note that all inner products are integers) be the maximum inner product over all vectors in $\cB$ for a fixed $a_i \in \cA$. We claim that $M_i$ can be recovered \emph{exactly} by
    \[
        M_i = \left\lfloor \frac{\log_n S_i'}{C} + 0.5 \right\rfloor.
    \]
    Note that each non-maximum term on a single row can be bounded by $0 < n^{C(a_i \cdot b_j)} \leq n^{CM_i}$, so we can bound the row sum by
    \[
        n^{CM_i} \leq S_i \leq n \cdot n^{CM_i} = n^{CM_i + 1}.
    \]
    Thus, applying $(1\pm \varepsilon)$-approximation to the upper and lower bounds respectively we get
    \[
    (1-\varepsilon) n^{CM_i} \leq S_i' \leq (1+\varepsilon)n^{CM_i + 1}.
    \]
    If we can show $M_i \leq (\log_n S_i') / C + 0.5 < M_i + 1$ then we are done.
    Indeed, using our definition for $C$ we get
    \[
    \frac{\log_n S_i'}{C} + 0.5 \leq \frac{CM_i+ 1 + \log_n(1+\varepsilon)}{C} + 0.5 = M_i + \frac{1 + \log_n(1+ \varepsilon)}{C} + 0.5 < M_i + 1,
    \]
    and 
    \[
    \frac{\log_n S_i'}{C}  + 0.5 > \frac{CM_i + \log_n(1 - \varepsilon)}{C} + 0.5 = M_i + \frac{\log_n (1 - \varepsilon)}{C} + 0.5 > M_i.
    \]
\end{proof}

Combining the above reduction with the conditional lower bound for $\mip$ (\Cref{thm:z-mip-lower-bound}), we obtain \Cref{m-thm:super-const-d-lb}.

\begin{restatable}[Formal \Cref{m-thm:super-const-d-lb}]{theorem}{SuperConstLB}
    \label{thm:super-const-d-lb}
    Fix $\varepsilon = \Theta(1)$ and $B = \poly(n)$.
    For all $\delta > 0$, there exists $C = C(\delta)$ and $d = 2^{C \log^* n}$ such that any algorithm computing $\attc(n, d, \entryBound, \varepsilon)$ requires $n^{2 - \delta}$ time under $\SETH$.
\end{restatable}

\subsection{Lower Bound for Attention with Large Head Dimension}

In this section, we study the case of large head dimension where \( d = \poly(n) \). Through a reduction from the OV problem, we show that computing \( \aattc(n, d, B, \varepsilon) \) requires explicitly computing the matrix product \( QK^\top \) when \( d = \poly(n) \), \( B = \bigO{\sqrt{\log n}} \), and \( \varepsilon = O(1) \) (additive approximation error). Furthermore, we establish a similar lower bound from the OV problem when \( d = \poly(n) \), \( B = O(1) \), and \( \varepsilon = \bigO{\frac{1}{\poly (n)}} \).

\begin{restatable}[Formal \Cref{m-thm:poly-d-lb}]{theorem}{OVLowerBound}
    \label{thm:poly-d-lb}
    Fix $d = \poly(n)$.
    There exists $B = O(\sqrt{\log n})$ and $\varepsilon = O(1)$ such that any algorithm computing $\attc(n, d, B, \varepsilon)$ requires $\TMUL(n, d, n)^{1 - o(1)}$ time under the Generalized High-Dimensional $\OV$ Hypothesis.

\end{restatable}

We show the following lemma to prove \Cref{thm:poly-d-lb}.

\begin{lemma}
    The $\OV$ problem can be computed exactly with one call to $\attc(n, d, B = O(\sqrt{\log n}), \varepsilon = O(1))$ and $O(nd)$ additional time.
\end{lemma}

\begin{proof} 
Let $\cA = \{a_1, \dotsc, a_n\}, \cB = \{b_1, \dotsc, b_n\} \subseteq \{0,1\}^d$ be two sets of vectors. We chose a constant $c$ such that $\varepsilon < c < 1$ and a constant $k$ such that $k < \frac{1 - c}{n(1+c)}$. We then define $Q, K \in \R^{n \times d}$: 
    \begin{equation}
        Q := -\sqrt{| \ln k |} \cdot \begin{bmatrix}
            \horzbar & a^{\top}_{1} & \horzbar \\
            \horzbar & a^{\top}_{2} & \horzbar \\
                     &  \vdots    &          \\
            \horzbar & a^{\top}_{n} & \horzbar
        \end{bmatrix}, \,
        K := \sqrt{|\ln k|}  \cdot \begin{bmatrix}
            \horzbar & b^{\top}_{1} & \horzbar \\
            \horzbar & b^{\top}_{2} & \horzbar \\
                     &  \vdots    &          \\
            \horzbar & b^{\top}_{n} & \horzbar
        \end{bmatrix}.
    \end{equation}

Due to \Cref{lem:de-normalization}, we can recover the row sums of $\exp(QK^\top)$ up to $\varepsilon$-multiplicative error in $\bigO{(\log \log n + \log(d B/\varepsilon)) \TATTC(n + 1, d + 1, B, \varepsilon)}$ time.
Let $S_i$ be the $(1 \pm \epsilon)$-approximation of the $i$-th row sum.

\[
    S_i := (1 \pm \epsilon)\sum_{j=1}^n e^{\ln(k)(a_i \cdot b_j)} = (1 \pm \epsilon)\sum_{j=1}^n k^{a_i \cdot b_j},
\]
which implies
 \[
     (1 - \epsilon)\sum_{j=1}^n k^{a_i \cdot b_j} \leq S_i \leq (1 + \epsilon)\sum_{j=1}^n k^{a_i \cdot b_j}.
 \]

If there are no orthogonal pairs of vectors in \( \mathcal{A} \) and \( \mathcal{B} \), then \( a_i \cdot b_j \) is a positive integer for all \( 1 \leq i, j \leq n \). Consequently, because \( 0 < k < 1 \), the maximum value of \( k^{a_i \cdot b_j} \) is \( k \). From this it follows that if there are no pairs of orthogonal vectors, all of the sums $S_i, \dotsc, S_n$ will be less than $1 - c$:

 \[
     S_i \leq (1 + \epsilon)\sum_{j=1}^n k^{a_i \cdot b_j} \leq (1 + \epsilon) nk < \frac{(1 + \epsilon)(1 - c)}{(1+c)} \leq \frac{(1 + c)(1 - c)}{(1+c)} = 1 - c.
 \]

 On the other hand, when there are one or more pairs of orthogonal vectors in $\cA$ and $\cB$, there will be at least one $k^{a_i \cdot b_j} = 1$ and a row sum $S_i$ will exist such that $S_i \geq 1 - c$:

 \[
     S_i \geq (1 - \epsilon)\sum_{j=1}^n k^{a_i \cdot b_j} > (1 - \epsilon)1 \geq 1 - c.
 \]

 By checking for the existence of a row sum $S_i$ that is greater than or equal to $1 - c$ we can determine whether there is a pair of orthogonal vectors in \( \mathcal{A} \) and \( \mathcal{B} \).
\end{proof}

We also show that when $d = \Theta(\log n)$, Attention is hard under $\SETH$ even with constant entry size $B$. 

\begin{restatable}{theorem}{OVLowerBoundConstant}
    \label{thm:ov-lb-constant}
    For all $\delta > 0$, there exists $C = C(\delta)$, $d = C \log n$ and $\varepsilon = n^{-C}$ such that any algorithm computing $\attc(n, d, \log 2, \varepsilon)$ requires $\Omega\left(n^{2 - \delta}\right)$ time under $\SETH$.
\end{restatable}

We show the following lemma to prove \Cref{thm:ov-lb-constant}.

\begin{lemma}
    The $\OV$ problem on vectors of dimension $d$ can be computed with high probability in time
    \begin{equation*}
        \bigtO{(\log n) (d + \log n) \TATTC \left( n + 1, d + 1, \log 2, \frac{1}{10 n 2^{d}} \right) }.
    \end{equation*}
\end{lemma}

Given the above lemma, suppose we have an algorithm computing $\attc$.
Given a $\delta$ define $\delta' = \delta / 2$ and let $C' = C'(\delta')$ and $d = C' \log n$ as required in \Cref{thm:ov-lower-bound}.
Then, let $\varepsilon = \frac{1}{10 n 2^{d}} = n^{-C}$ for some large constant $C = C(\delta) \geq C'$.
Any algorithm computing $\TATTC(n + 1, d + 1, \log 2, \varepsilon)$ then requires $\Omega(n^{2 - \delta})$ time, proving \Cref{thm:ov-lb-constant}.

\begin{proof}
    Let $\cA = \{a_1, \dotsc, a_n\}, \cB = \{b_1, \dotsc, b_n\} \subseteq \{0, 1\}^d$ be two sets of vectors.
    Define $Q, K \in \R^{n \times d}$ to be the matrices whose rows are formed by the vectors in $\cA$ and $\cB$, respectively, i.e.,
    \[
        Q := \log (2) \begin{bmatrix}
            \horzbar & a^{\top}_{1} & \horzbar \\
            \horzbar & a^{\top}_{2} & \horzbar \\
                     &  \vdots    &          \\
            \horzbar & a^{\top}_{n} & \horzbar
        \end{bmatrix}, \,
        K := \log (2) \begin{bmatrix}
            \horzbar & b^{\top}_{1} & \horzbar \\
            \horzbar & b^{\top}_{2} & \horzbar \\
                     &  \vdots    &          \\
            \horzbar & b^{\top}_{n} & \horzbar
        \end{bmatrix}.
    \]
    Note that
    \[
    QK^\top = 
        \begin{bmatrix}
            \log(2) \cdot a_1 \cdot b_1 & \dotsb & \log(2) \cdot a_1 \cdot b_n \\
            \vdots & \vdots & \ddots & \vdots \\
            \log(2) \cdot a_n \cdot b_1 & \dotsb & \log(2) \cdot a_n \cdot b_n
        \end{bmatrix}.
    \]
    and the $i$-th row sum of $\exp(QK^{\top})$ is given by $\sum_{j = 1}^{n} 2^{a_{i} \cdot b_{i}}$.
    In particular, note that all row sums are integers satisfying $n \leq S_{i} \leq n 2^{d}$.
    From \Cref{lem:de-normalization}, we can recover the row sums up to $\frac{1}{10 n 2^{d}}$ and therefore $\frac{1}{10}$-additive error in time
    \begin{equation*}
        \bigO{(\log \log n + \log (d n 2^{d})) \TATTC \left( n + 1, d + 1, \log 2, \frac{1}{10 n 2^{d}} \right) }.
    \end{equation*}
    Given the $\frac{1}{10}$-additive approximation of $S_{i}$, we may recover $S_{i}$ by rounding since they are integers.
    Note that $S_{i} \leq n 2^{d}$ and can therefore be represented in $O(d + \log n)$ bits.
    
    If there are no orthogonal pairs of vectors in $\cA$ and $\cB$, then $a_i \cdot b_j$ is a positive integer for all $1 \leq i, j \leq n$, which means $2^{a_i \cdot b_j}$ is an even number. It follows that all of the sums $S_1, \dotsc, S_n$ are also even numbers.

    Conversely, when an orthogonal pair of vectors exists in $\cA$ and $\cB$, we would like to detect this based on the sums $S_1, \dotsc, S_n$ as well. 
    Note that when $a_i \cdot b_j = 0$ we have $2^{a_i \cdot b_j} = 1$, which may potentially make the sum into an odd number.
    However, when there are an even number of such orthogonal pairs, the sum remains even, and we cannot distinguish from the previous case. 
    The workaround is to use a standard sampling method, so that with high probability, we include exactly one pair of orthogonal vectors in the sample, and therefore the corresponding sum will be odd.
    
    Fix an index $1 \leq i \leq n$ such that $a_i \in \cA$ is orthogonal to some vector in $\cB$.
    Let $b^*$ be the last vector in $\cB$ orthogonal to $a_{i}$.
    Without loss of generality, we may assume that the zero vector $\mathbf{0}_d \notin \cA$, since we can check this in $O(nd)$ time and immediately accepts the input if this is the case.
    Given $\mathbf{0}_d \notin \cA$, we know that vector $\mathbf{1}_d$ is not orthogonal to any vector in $\cA$. 
    Consider the following sampling procedure:

    Construct $\cB'$ by including each vector of $\cB$ with probability $\frac{1}{2}$ independently and padding with $\mathbf{1}_{d}$ to ensure $\cB'$ has $n$ vectors.
    Note that with probability exactly $\frac{1}{2}$ we have that $\cB'$ contains an odd number of orthogonal vectors to $a_{i}$ (i.e. $b^*$ is included with probability $\frac{1}{2}$).
    In particular, sampling $\cB'$ $O(\log n)$-times allows us to detect an odd row sum with high probability.

    Thus, the overall algorithm requires involves $O(\log n)$ loops, where in each loop we check for an odd row-sum using \Cref{lem:de-normalization}.
    The overall time is therefore
    \begin{equation*}
        \bigtO{(\log n) (d + \log n) \TATTC \left( n + 1, d + 1, \log 2, \frac{1}{10 n 2^{d}} \right) }.
    \end{equation*}
\end{proof}

\section{Conclusion}

We conclude with some open questions.
The most natural question is settling the complexity of $\mip$ when $1 \ll d \ll 2^{\Theta(\log^* n)}$. We have shown several conditional lower bounds for Attention computation.
Is Attention fine-grained equivalent to any well-studied problem?
If such a relationship can be established, then breakthroughs on well-studied problems in fine-grained complexity can lead to breakthroughs on Attention computation.
While this work focuses on characterizing the complexity of training a single Attention unit, the complexity of computing a full transformer remains open: perhaps the cost of computing many Attention units is less than computing each of them separately.

\section{Acknowledgments}

The authors would like to thank Josh Alman for helpful discussions and pointing us to the Fast Multipole Method.

\bibliographystyle{alpha}
\bibliography{main}

\end{document}